%% file: arxiv_version.tex
\documentclass{article} % For LaTeX2e
\usepackage{iclr2027_conference,times}

\input{math_commands.tex}

\usepackage{hyperref}
\usepackage{url}

\iclrfinalcopy

\usepackage{algorithm}
\usepackage{algpseudocode}
\usepackage{graphicx}
\usepackage{dsfont}
\usepackage{bm}
\usepackage{booktabs}

\usepackage{amsmath,amsthm,amssymb}
\usepackage{mathrsfs}
\usepackage{arydshln} % 加载虚线支持包
\usepackage{subcaption}
\usepackage{wrapfig}
\usepackage{multirow}
\usepackage{enumitem} % 控制itemize行宽
\usepackage[table]{xcolor} % 用于表格颜色
\usepackage[most]{tcolorbox}

\definecolor{gptbar}{RGB}{230,232,236}
\definecolor{qwenbar}{RGB}{228,220,246}
\definecolor{my_approach}{RGB}{234,242,255}
\definecolor{my_approach_heavy}{RGB}{203,216,230}

\usepackage{xcolor}
\usepackage{fancyvrb}
\usepackage{fvextra}

\tcbset{
  promptbox/.style={
    enhanced,
    breakable,
    colback=white,
    colframe=blue!35!black,
    colbacktitle=blue!15,
    coltitle=black,
    fonttitle=\bfseries,
    boxrule=1.2pt,
    arc=2mm,
    left=2mm,
    right=2mm,
    top=2mm,
    bottom=2mm,
  }
}

\theoremstyle{plain}
\newtheorem{theorem}{Theorem}[section]
\newtheorem{proposition}[theorem]{Proposition}

\theoremstyle{definition}

\theoremstyle{remark}

\title{ERSkill: Evolving for Skill-Guided Adaptive Memory Retrieval}

\author{
\textbf{Haolong Chen}$^{1,2,3}$, \textbf{Liang Zhang}$^{4}$, \textbf{Zhuo Li}$^{1,2,3}$, \textbf{Lei Xue}$^{4}$, \textbf{Guangxu Zhu}$^{1,2,3,5}$ \\
$^1$Shenzhen International Center for Industrial and Applied Mathematics \\
$^2$Shenzhen Research Institute of Big Data \\
$^3$The Chinese University of Hong Kong, Shenzhen \\
$^4$Shenzhen Campus of Sun Yat-sen University \\
$^5$Shenzhen Loop Area Institute \\
\texttt{haolongchen1@link.cuhk.edu.cn,} 
\texttt{zhangliang27@mail.sysu.edu.cn,} \\
\texttt{221019088@link.cuhk.edu.cn,} 
\texttt{xuelei3@mail.sysu.edu.cn,} 
\texttt{gxzhu@sribd.cn}
}

\begin{document}

\maketitle

% 把 ICLR 的 "Published as a conference paper..." 改掉
\fancyhead[L]{Preprint}

\begin{abstract}
While Large Language Model (LLM) agents increasingly rely on long-term memory for persistent interactions, the retrieval mechanisms governing this memory are rarely treated as evolvable components. This static approach limits performance on heterogeneous memory queries, which often demand diverse evidence construction strategies. To address this, we introduce \textbf{ERSkill}, a retrieval-centric framework for evolving, skill-guided memory access. ERSkill compiles interaction histories into a structured memory store and represents retrieval behaviors as executable skills composed of fundamental primitives. At inference time, a trained router dynamically matches each query to a suitable retrieval skill to construct tailored evidence for answer generation. ERSkill co-evolves the skill set and the router during training. It employs an experience trie to efficiently record explored retrieval paths, alongside a double-frontier mechanism that separates oracle-side capability expansion from router-validated deployment. Experiments across multiple agent memory benchmarks demonstrate that ERSkill substantially outperforms strong non-evolving and evolving baselines. Notably, it improves the overall average across F1, BLEU-1, and LLM-judge scores by 31.3\% with Qwen3-Next-80B-A3B-Instruct and by 21.4\% with GPT-5.4-nano.
\end{abstract}

\input{section/intro}

\input{section/method}

\input{section/experiment}

\input{section/related}

\input{section/conclusion}

\subsection*{AI use statement}
In this work, we used LLMs to assist with code generation and language polishing.
We have reviewed all AI-assisted work.
We take responsibility for the final content of this work, including text, claims or artifacts produced with the aid of generative AI.

\subsection*{Reproducibility statement}
We provided a detailed description and implementation of our method in Section~\ref{sec:exp-setup}, Appendix~\ref{appen:details_for_memory_storage}, Appendix~\ref{appen:details_for_evolution}, and Appendix~\ref{appen:imple_details}.

% (This section is \textbf{recommended} and does not count toward the page limit.)

% It is important that the work published in ICLR is reproducible. Authors are
% strongly encouraged to include a paragraph-long Reproducibility Statement at the
% end of the main text (before references) to discuss the efforts that have been
% made to ensure reproducibility. This paragraph should not itself describe
% details needed for reproducing the results, but rather reference the parts of
% the main paper, appendix, and supplemental materials that will help with
% reproducibility. For example, for novel models or algorithms, a link to an
% anonymous downloadable source code can be submitted as supplementary materials;
% for theoretical results, clear explanations of any assumptions and a complete
% proof of the claims can be included in the appendix; for any datasets used in
% the experiments, a complete description of the data processing steps can be
% provided in the supplementary materials. Each of the above are examples of
% things that can be referenced in the reproducibility statement.



\bibliography{iclr2027_conference}
\bibliographystyle{iclr2027_conference}

\input{section/appendix}

\end{document}

%% file: math_commands.tex
\usepackage{amsmath,amsfonts,bm}

\def\eqref#1{equation~\ref{#1}}
\def\1{\bm{1}}

\DeclareMathAlphabet{\mathsfit}{\encodingdefault}{\sfdefault}{m}{sl}
\SetMathAlphabet{\mathsfit}{bold}{\encodingdefault}{\sfdefault}{bx}{n}

%% file: section/intro.tex
\section{Introduction}

\begin{wrapfigure}{r}{0.5\textwidth}
    % \vspace{-1.2em}
    \vspace{-2.2em} % arxiv
    \centering
    \includegraphics[width=0.5\textwidth]{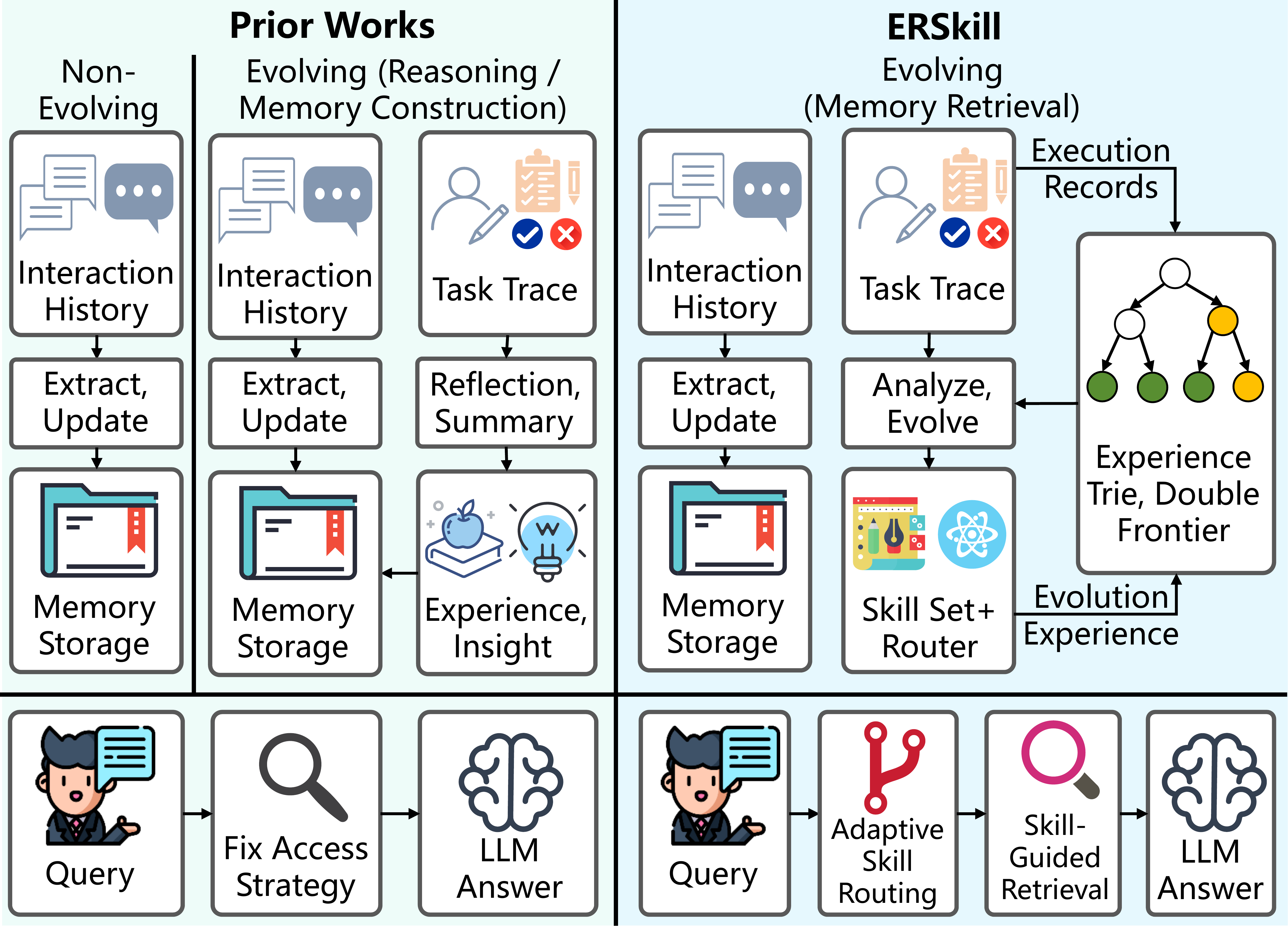}
    \vspace{-1.4em}
    \caption{
    \textit{Left}: Prior works include non-evolving methods and evolving methods for reasoning or memory construction.
    \textit{Right}: ERSkill studies evolving memory retrieval through skill-guided adaptive access.
    }
    \label{fig:paradigm}
    \vspace{-1.2em}
\end{wrapfigure}

Large Language Model (LLM)~\citep{openai2024gpt4technicalreport,geminiteam2025geminifamilyhighlycapable,deepseekai2025deepseekv3technicalreport,yang2025qwen3} agents are increasingly expected to function as persistent collaborators rather than one-shot assistants~\citep{hu2025memory, chen2025first, chen2026error, chen2026overview, li2026qoereasoner}.
Over extended interactions, an agent accumulates user preferences, tracks changing events, remembers prior decisions, and reuses experiences, motivating rapid progress in agent memory~\citep{huang2026rethinking}.
Recent work has made progress toward persistent agents by building explicit external memory to store interaction history.
These systems construct and maintain memory through operations such as information extraction, compression, updating, and forgetting~\citep{xu2025mem,kang2025memory,fang2025lightmem}, and retrieve stored memories as evidence for downstream tasks.
More recent work further moves toward evolving agents: instead of only storing key information, agents reflect on past trajectories, distill reusable experiences or insights, store and apply them to guide future behavior~\citep{zhao2024expel,suzgun2026dynamic,ouyang2025reasoningbank,agrawal2025gepa,zhang2026memskill}.

Despite this progress, the retrieval side of agent memory remains underexplored as an object of evolution.
Existing evolving methods mainly use past task experience to improve future reasoning~\citep{zhao2024expel,suzgun2026dynamic,ouyang2025reasoningbank,agrawal2025gepa}; for example, ReasoningBank~\citep{ouyang2025reasoningbank} reflects on past task traces to distill reusable reasoning experiences.
MemSkill~\citep{zhang2026memskill} evolves LLM-based memory extraction skills, but the resulting memories are still accessed through a predefined dense retrieval strategy at query time.
Thus, while memory content and reasoning guidance become increasingly adaptive, retrieval behavior itself often remains predefined.
This becomes limiting for agent memory question answering, where queries can differ substantially in their information demands.
For example, ``What gift did Alice buy for Bob during her Hawaii trip?'' primarily requires retrieving a specific event, whereas ``Why did Alice later stop planning another Hawaii trip with Bob?'' requires connecting earlier events with later developments to uncover causal relations.
Such queries require qualitatively different evidence construction behaviors, raising a central question about query-time memory access:

\begin{tcolorbox}[
  colback=my_approach,
  colframe=my_approach_heavy,
  colbacktitle=gray!35,
  coltitle=black,
  boxrule=1pt,
  arc=5pt,
  drop shadow,
  parbox=false,
  before skip=3pt,
  after skip=3pt,
  left=3pt,
  right=3pt,
  top=3pt,
  bottom=3pt,
  boxsep=1pt,
]
\textit{How can an LLM agent evolve and learn to compose complex retrieval actions so that its retrieval behavior adapts to the heterogeneous information demands of different queries?}
\end{tcolorbox}

In this paper, we propose \textbf{ERSkill} (Evolving Retrieval Skill), a retrieval-centric evolving framework that models agent memory access as skill-guided~\citep{anthropic2025agentskills} evidence construction.
To support diverse retrieval perspectives, ERSkill first builds a retrieval-oriented memory storage that exposes memory via a library of retrieval primitives, such as dense retrieval~\citep{karpukhin2020dense}, BM25 retrieval~\citep{robertson2009probabilistic}, and query rewriting~\citep{ma2023query}.
Based on these primitives, ERSkill represents retrieval behavior as a set of executable retrieval skills.
Each skill specifies a concrete retrieval schema by composing primitives into a sequence.
This skill abstraction makes retrieval behavior reusable, interpretable, and refinable.
To choose appropriate retrieval behaviors for each query, ERSkill trains a skill router that matches the query's information demand to skills.
At inference time, the selected skill controls how memory atoms are gathered, expanded, and organized into a task-specific evidence view for answer generation.
In this way, ERSkill realizes query-adaptive memory access while keeping the retrieval process executable and interpretable.

However, a predefined static set of retrieval skills limits the expressiveness of memory access, as retrieval requirements vary across queries.
We therefore design a skill-router co-evolution procedure that jointly updates the skill set and the router by analyzing previous task traces and identifying gaps between current skill behavior and desired evidence construction.
To reduce inefficient exploration, ERSkill stores explored primitive paths in an experience trie, allowing the skill generator to reuse past rollout experience and avoid repeatedly proposing equivalent retrieval programs.
Moreover, because the router may not always select the best skill during deployment, skill evolution should consider not only whether a skill is useful in principle, but also whether its utility can be reliably activated by the router.
To this end, we use a Pareto-style double-frontier mechanism: the capability frontier maintains a compact set of skills with the best retrieval capabilities, while the deploy frontier maintains a router-facing skill set validated under routed inference.
This allows ERSkill to expand retrieval capability while gating router-facing updates according to routed validation performance.
Our contributions are summarized as follows:
\begin{itemize}[itemsep=0em, topsep=0em, leftmargin=1em]
    \item We introduce a retrieval-centric framework for agent memory, treating memory access as query-adaptive evidence construction rather than a fixed retrieval strategy.
    \item We propose ERSkill, a skill-guided framework that co-evolves retrieval skills and the router using an experience trie and a double-frontier design.
    \item Experiments on agent memory benchmarks show that ERSkill outperforms strong baselines, improving the overall average across F1, BLEU-1, and LLM-as-a-Judge scores by 31.3\% with Qwen3-Next-80B-A3B-Instruct and by 21.4\% with GPT-5.4-nano. ERSkill also exhibits a favorable LLM-token/quality trade-off.
\end{itemize}

\vspace{-0.8em}

%% file: section/method.tex
\section{Methodology}

\begin{figure*}[t]
    \centering
    \includegraphics[width=\textwidth]{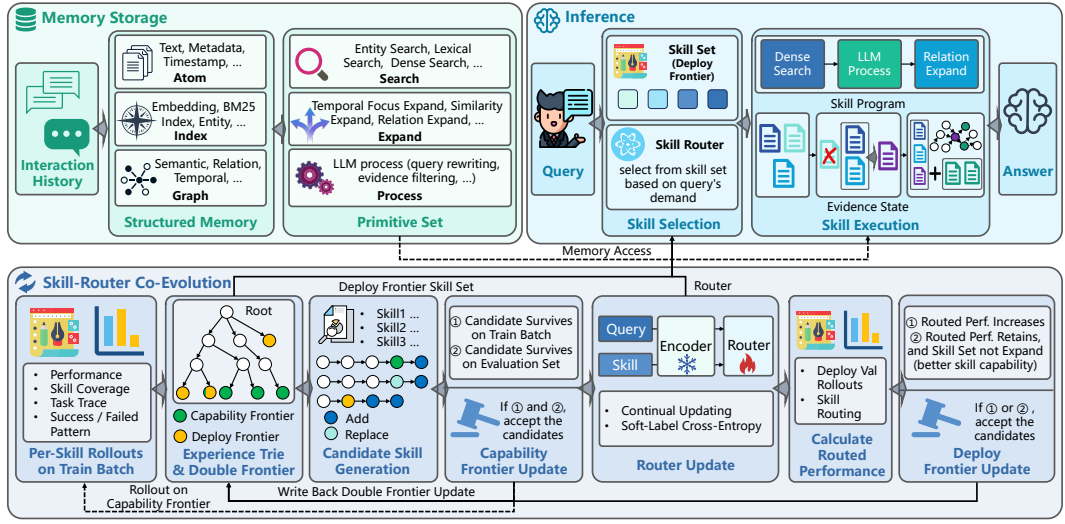}
    \caption{
    Overview of ERSkill.
    ERSkill compiles interaction history into memory storage, builds retrieval skills from primitives, and selects a query-matched skill to construct evidence for answering.
    During training, the skills and the router co-evolve, with an experience trie recording explored paths and double frontiers supporting skill capability expansion and router-facing deployment.
    }
    \label{fig:overview}\vspace{-0.5em}
\end{figure*}

% As shown in Figure~\ref{fig:overview}, \textbf{ERSkill} adapts memory retrieval by selecting, for each query, a retrieval skill that matches its information demand.
% It first compiles the interaction history into memory storage, then uses a trained router to select the most suitable skill, and finally executes the selected skill to construct evidence for answering.
% During training, the skills and the router co-evolve.

As shown in Figure~\ref{fig:overview}, ERSkill compiles interaction
history into memory storage and routes each query to a retrieval skill
matching its information demand for evidence construction.
During training, the skills and router co-evolve.

\subsection{Structured Memory Storage}

Given interaction history $D$, ERSkill constructs a structured memory
$M(D)=(\mathcal{A},\mathcal{I},\mathcal{G})$, where $\mathcal{A}$ contains
atom-level records with text, metadata, and timestamps, $\mathcal{I}$
provides indexes for atom search, and $\mathcal{G}$ provides graphs for
atom expansion.
Appendix~\ref{appen:structured_memory} gives the construction details.

\paragraph{Retrieval primitive.}
ERSkill exposes the memory $M(D)$ through a primitive library
$\mathcal{P}=\{p_1,\ldots,p_m\}$, where each primitive updates the
evidence state as
$p:(q,s,M(D))\mapsto s'$.
\textit{Search primitives} include \texttt{entity\_search},
\texttt{lexical\_search}, and \texttt{dense\_search}, which access memory
through entity, BM25, and embedding-based indexes, respectively.
\textit{Expansion primitives} include \texttt{temporal\_focus\_expand},
\texttt{similarity\_expand}, and \texttt{relation\_expand}, which expand
evidence using temporal, similarity, and relation structure.
We further include \texttt{llm\_process} for query rewriting, evidence
filtering, and other dynamic control.
The primitive library is fixed during evolution; skills differ in how
these shared primitives are composed, allowing their execution outcomes
to be accumulated in a common experience structure.

\begin{figure*}[t]
    \centering
    \includegraphics[width=\textwidth]{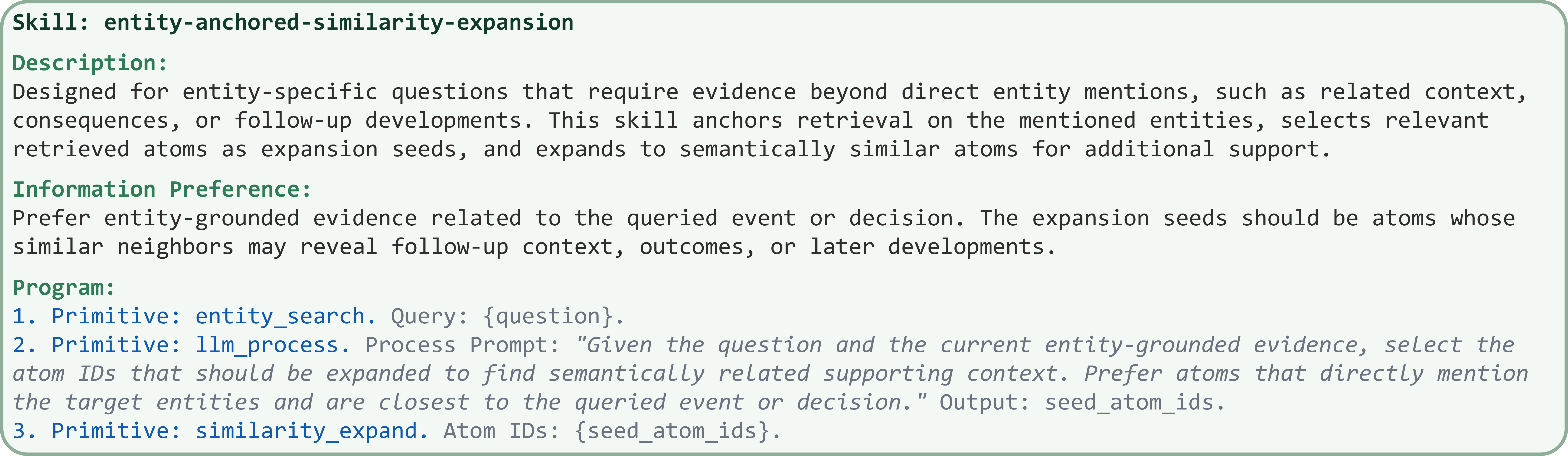}
    \caption{
    Example of an entity-anchored similarity expansion skill, combining
    entity retrieval, seed selection, and similarity-based expansion.
    }
    \label{fig:skill_sample}\vspace{-0.5em}
\end{figure*}

\subsection{Inference}

\paragraph{Retrieval skills.}
A retrieval skill is an executable primitive program
$\kappa=(c_\kappa,\rho_\kappa)$, where $c_\kappa$ contains its description
and information preference, and
$\rho_\kappa=(p_{\kappa,1},\ldots,p_{\kappa,L_\kappa})$ is its primitive
sequence.
Given query $q$, the skill sequentially updates the evidence state as
$s_j=p_{\kappa,j}(q,s_{j-1},M(D))$ and returns $s_{L_\kappa}$ as the
evidence view.
Figure~\ref{fig:skill_sample} shows an example skill.

\paragraph{Skill routing and answer generation.}
Given $q$ and $\mathcal{K}$, ERSkill encodes the query and each skill
with a shared encoder,
$h_q=\mathrm{Enc}(q)$ and $h_\kappa=\mathrm{Enc}(\kappa)$,
scores their compatibility with $u_\theta(q,\kappa)$, and defines the router
$R_\theta(\kappa\mid q,\mathcal{K})
=\frac{\exp(u_\theta(q,\kappa))}
{\sum_{\kappa'\in\mathcal{K}}\exp(u_\theta(q,\kappa'))}$.
The frozen text encoder represents each skill from its description, information preference, and program, allowing newly evolved skills to be scored without changing the router's output space.
At inference,
$\hat{\kappa}=\pi_\theta(q;\mathcal{K})
=\arg\max_{\kappa\in\mathcal{K}}R_\theta(\kappa\mid q,\mathcal{K})$;
ERSkill executes $\hat{\kappa}$ and generates the final answer
$\hat{y}=f_{\mathrm{LLM}}(q,s_{L_{\hat{\kappa}}})$.

\subsection{Skill-Router Co-Evolution}

ERSkill maintains two Pareto-style frontiers.
The \emph{capability frontier} $\mathcal{C}_t$ retains compact,
non-redundant skills with oracle-side value, tracking the best attainable
retrieval capability without routing errors.
The \emph{deploy frontier} $\mathcal{B}_t$ is the router-facing skill set,
whose updates are gated by routed validation performance.
Together, they separate capability expansion from deployment.
At evolution step $t$, ERSkill maintains an experience trie
$\mathcal{T}_t$, frontiers $\mathcal{C}_t$ and $\mathcal{B}_t$, and
router parameters $\theta_t$.
Both frontiers are initialized with three single-primitive skills based
on \texttt{dense\_search}, \texttt{entity\_search}, and
\texttt{lexical\_search}, so $\mathcal{B}_0=\mathcal{C}_0$.
Given a training batch $\mathcal{Q}_t$, each skill in $\mathcal{C}_t$ is
run on every query, with performance, execution traces, and skill
overlaps recorded in $\mathcal{T}_t$ to guide candidate generation,
frontier updates, and router training.
Appendix~\ref{appen:algo_evolution} gives the full algorithm.

\paragraph{Skill candidate generation.}
ERSkill generates candidates by editing primitive paths from the current
capability frontier.
To reuse search experience and avoid duplicate exploration, it maintains
an \emph{experience trie} $\mathcal{T}$ over primitive sequences, where
shared prefixes are stored once and each explored path records its
train/validation rollouts and frontier status.\footnote{We treat
\texttt{llm\_process} as one primitive type despite prompt variations.}
The generator uses skill summaries, rollout statistics, success/failure
patterns, and trie history to propose new paths while excluding those
already explored.
Thus, both accepted and rejected candidates inform later evolution.
Appendix~\ref{appen:skill_cand_gen} details candidate generation, and
Appendix~\ref{appen:experience_trie_example} shows an example trie.

\paragraph{Capability frontier update.}
ERSkill retains skills with unique oracle-side value.
For each query--skill pair, let $r(q,\kappa)\in[0,1]$ denote its
performance, instantiated as backbone-LLM judge accuracy in our
experiments.
For a batch $\mathcal{Q}$, define
$\mathrm{Util}(\kappa;\mathcal{Q})
=\frac{1}{|\mathcal{Q}|}\sum_{q\in\mathcal{Q}}r(q,\kappa)$ and
$g_{\mathcal{K}}(q)=\max_{\kappa\in\mathcal{K}}r(q,\kappa)$.
The recomputation operator $\Phi(\mathcal{K};\mathcal{Q})$ sorts skills
by utility and prunes a skill only if its removal leaves
$g_{\mathcal{K}}(q)$ unchanged for every query, thereby preserving
oracle performance while removing redundancy.
Given candidates $\mathcal{U}_t$, ERSkill first computes
$\widetilde{\mathcal{C}}_{t+1}
=\Phi(\mathcal{C}_t\cup\mathcal{U}_t;\mathcal{Q}_t)$ and retains
$\mathcal{V}_t=\mathcal{U}_t\cap\widetilde{\mathcal{C}}_{t+1}$.
After evaluating $\mathcal{V}_t$ on validation data, it updates
$\mathcal{C}_{t+1}
=\Phi(\mathcal{C}_t\cup\mathcal{V}_t;\mathcal{Q}_{\mathrm{val}})$.

\paragraph{Router update.}
ERSkill trains the router from a window $\mathcal{W}_t$ of rollout
instances $(q,\mathcal{K}_q,\{r(q,\kappa)\}_{\kappa\in\mathcal{K}_q})$,
where $\mathcal{K}_q$ is the skill set used for that rollout.
For a mini-batch $\mathcal{M}\subset\mathcal{W}_t$, rollout scores
define temperature-scaled soft targets, and the router is optimized
by cross-entropy:
\begin{equation}
\tilde{p}(\kappa | q,\mathcal{K}_q)
=
\frac{\exp(r(q,\kappa)/\tau)}
{\sum_{\kappa' \in \mathcal{K}_q}\exp(r(q,\kappa')/\tau)},
\mathcal{L}_{\mathrm{router}}
=
-\sum_{(q,\mathcal{K}_q,\cdot)\in\mathcal{M}}
\sum_{\kappa\in\mathcal{K}_q}
\tilde{p}(\kappa | q,\mathcal{K}_q)
\log R_\theta(\kappa | q,\mathcal{K}_q).\nonumber
\end{equation}

\paragraph{Deploy frontier update.}
When the capability frontier changes, ERSkill considers newly retained
candidates $\mathcal{H}_t=\mathcal{U}_t\cap\mathcal{C}_{t+1}$ and forms
$\mathcal{B}'_t=\Phi(\mathcal{B}_t\cup\mathcal{H}_t;
\mathcal{Q}_{\mathrm{val}})$.
For any skill set $\mathcal{K}$, define routed performance as
$\mathrm{Routed}(\mathcal{K},\theta;\mathcal{Q})=\frac{1}{|\mathcal{Q}|}\sum_{q\in\mathcal{Q}}r(q,\pi_\theta(q;\mathcal{K}))$.
Let
$\Delta_{\mathrm{route}}
=
\mathrm{Routed}(\mathcal{B}'_t,\theta_{t+1};\mathcal{Q}_{\mathrm{val}})
-
\mathrm{Routed}(\mathcal{B}_t,\theta_{t+1};\mathcal{Q}_{\mathrm{val}})$.
ERSkill accepts $\mathcal{B}'_t$ if
$\Delta_{\mathrm{route}}\ge\gamma_{\mathrm{route}}$, or if
$\Delta_{\mathrm{route}}\ge-\xi_{\mathrm{drop}}$ and
$|\mathcal{B}'_t|\le|\mathcal{B}_t|$; otherwise it retains
$\mathcal{B}_t$.
The outcome is written to the experience trie, and the final deploy
frontier is used for inference.

\begin{proposition}[Oracle-coverage monotonicity of the two-frontier update]
\label{prop:oracle-safe-frontier}
Denote the oracle coverage as $\mathrm{OCov}(\mathcal{K};\mathcal{Q})=\frac{1}{|\mathcal{Q}|}\sum_{q\in\mathcal{Q}}g_{\mathcal{K}}(q)$.
For every evolution step $t$, ERSkill satisfies $\mathrm{OCov}(\mathcal{C}_{t+1};\mathcal{Q}_{\mathrm{val}})\geq \mathrm{OCov}(\mathcal{C}_t;\mathcal{Q}_{\mathrm{val}})$ and $\mathrm{OCov}(\mathcal{B}_{t+1};\mathcal{Q}_{\mathrm{val}})\geq \mathrm{OCov}(\mathcal{B}_t;\mathcal{Q}_{\mathrm{val}})$.
\end{proposition}

Proposition~\ref{prop:oracle-safe-frontier} shows that both frontiers
preserve non-decreasing oracle coverage on validation data, while deploy
updates are additionally gated by routed performance.
The proof is in Appendix~\ref{appen:proof}.

%% file: section/experiment.tex
\input{table/main_comparision}

\section{Experiments}
\label{sec:experiments}

\subsection{Experimental Setup}
\label{sec:exp-setup}

\paragraph{Benchmarks.}
We evaluate ERSkill on LoCoMo~\citep{maharana2024evaluating},
LongMemEval~\citep{wu2024longmemeval}, and
PerLTQA~\citep{du2024perltqa}, covering multi-session conversational and
heterogeneous-source memory.
See Appendix~\ref{appen:benchmarks} for details.

\paragraph{Baselines.}
We compare against three groups of baselines:
(1) \textit{Non-Evolving}: A-Mem~\citep{xu2025mem},
MemoryOS~\citep{kang2025memory}, and LightMem~\citep{fang2025lightmem};
(2) \textit{Adaptive Retrieval}: Adaptive-RAG~\citep{jeong2024adaptive};
and (3) \textit{Evolving}: Dynamic Cheatsheet~\citep{suzgun2026dynamic},
ReasoningBank~\citep{ouyang2025reasoningbank},
GEPA~\citep{agrawal2025gepa}, and MemSkill~\citep{zhang2026memskill}.
The first three evolving methods distill reusable experience from task
traces and use standard RAG memory~\citep{karpukhin2020dense}, while MemSkill evolves memory-construction skills.

\paragraph{Implementation details.}
We evaluate with \texttt{Qwen3-Next-80B-A3B-Instruct}~\citep{yang2025qwen3}
and \texttt{GPT-5.4-nano}~\citep{openai2026gpt54mini}, using
\texttt{GPT-4o-mini} as the evaluation judge.
We report F1, BLEU-1 (B1), and LLM-judge (L-J) scores
(Appendix~\ref{appen:llm_judge_prompt}).
ERSkill uses train batch sizes of 20 and 40 for LoCoMo and PerLTQA,
respectively, and is trained for one epoch.
We instantiate the router encoder $\mathrm{Enc}(\cdot)$ with
\texttt{Qwen3-Embedding-0.6B}~\citep{zhang2025qwen3}; projected query and
skill embeddings are concatenated and scored by a two-layer MLP.
For dense retrieval, all methods use
Contriever~\citep{izacard2021unsupervised}.
See Appendix~\ref{appen:imple_details} for details.

\subsection{Comparison Experiments}
\label{sec:comparison}

\begin{wrapfigure}{r}{0.38\textwidth}
    \vspace{-2.6em}
    \centering
    \includegraphics[width=0.38\textwidth]{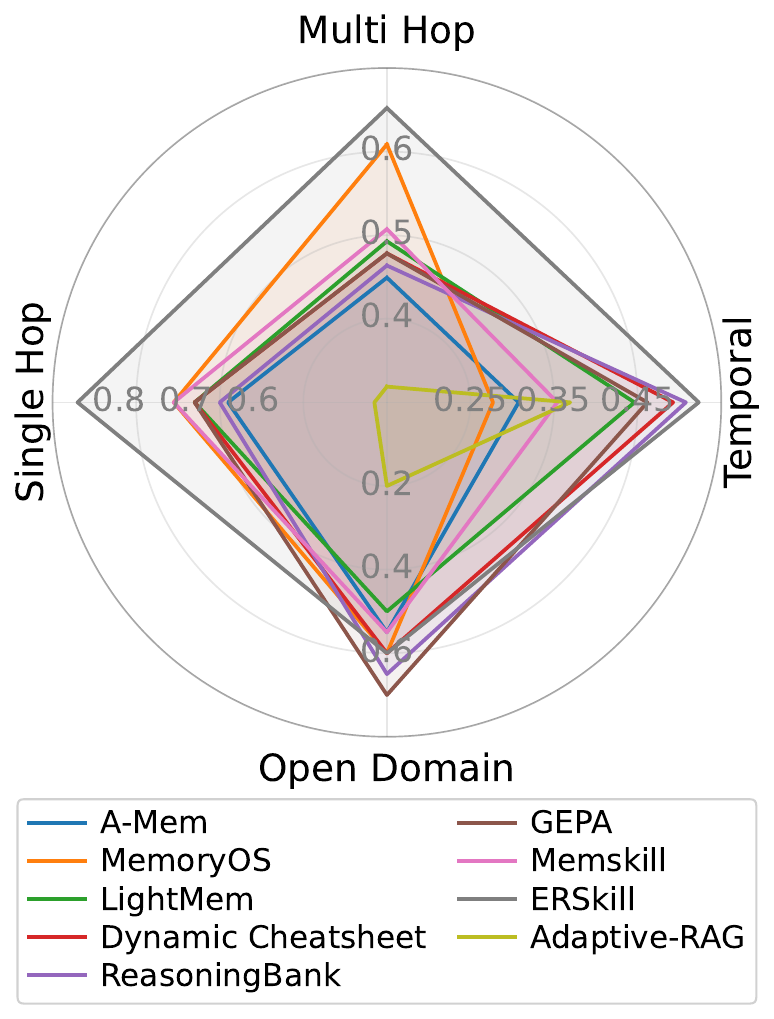}
    \vspace{-1.2em}
    \caption{Fine-grained comparison on LoCoMo categories (GPT-5.4-nano, L-J). \textbf{ERSkill shows larger gains on evidence-intensive tasks such as Single Hop and Multi Hop}.}
    \label{fig:radar_comparison}
    \vspace{-3.4em}
\end{wrapfigure}

\paragraph{ERSkill achieves the strongest overall performance.}
Table~\ref{tab:main-results} shows that ERSkill achieves the best overall
average under both backbones, with relative improvements of 31.3\% and
21.4\% for Qwen3-Next-80B-A3B-Instruct and GPT-5.4-nano, respectively
(Appendix~\ref{app:relative_improvement}).
Unlike Adaptive-RAG, which adapts a fixed retrieval policy, ERSkill also
evolves the executable behaviors available to the router.
Figure~\ref{fig:radar_comparison} shows larger gains on Single Hop and
Multi Hop, consistent with improved evidence localization.

\paragraph{ERSkill transfers effectively across datasets.}
On LongMemEval, ERSkill directly reuses the router and retrieval skills
trained on LoCoMo without additional training and remains the
best-performing method under both backbones, suggesting that the learned
skills capture reusable retrieval behaviors.

\input{table/comparison_shared_memory_small}

\paragraph{Controlled comparison under a shared memory substrate.}
To separate memory representation from retrieval policy, we replace
the original RAG memory of Dynamic Cheatsheet, ReasoningBank, and GEPA
with ERSkill's multi-primitive substrate while keeping their evolution
procedures unchanged.
Table~\ref{tab:same-memory} shows that the shared substrate yields
mixed gains across methods and backbones, while ERSkill remains
substantially stronger than all controlled variants.
Thus, richer retrieval primitives alone do not explain ERSkill's gains;
learning how to compose and route them is also important.

\input{table/cost_performance}

\paragraph{LLM-token efficiency.}
Table~\ref{tab:end-to-end-cost} reports LLM-token cost for memory
construction, evolution, and inference.
Evolution cost includes trace analysis, skill generation, rollouts,
answer generation, and judge calls, while cached rollout reuse,
frontier operations, and router updates incur no additional LLM-token
cost (Appendix~\ref{appen:rollout_cache}).
ERSkill achieves the strongest answer quality with low memory-construction
and inference costs, while its evolution cost remains competitive.
Overall, ERSkill exhibits a favorable LLM-token/quality trade-off.

\subsection{Where Do the Gains Come From?}
\label{sec:ablation}

\input{table/ablation}

\paragraph{Ablation study.}
We evaluate four component removals and a random-mutation variant.
\textit{w/o Skill Evolution} keeps the seed skills fixed;
\textit{w/ Random Skill Mutation} replaces trace-grounded proposals with
random primitive mutations;
\textit{w/o Router} uses LLM-based skill selection;
\textit{w/o Double Frontier} accepts all generated candidates; and
\textit{w/o Experience Trie} generates candidates without historical
path-level records.
As shown in Table~\ref{tab:ablation}, the full ERSkill outperforms all
variants, with the largest drops from removing skill evolution and the
router.
The double frontier and experience trie provide additional gains.

\input{table/retrieval_evidence_analysis}

\paragraph{Retrieval and evolution behavior.}
Table~\ref{tab:evidence-quality} shows that skill evolution improves
gold-evidence coverage and downstream answer quality over the fixed seed
catalog.
Under the same primitive library and number of evolution batches,
Table~\ref{tab:evolution-strategy} shows that trace-grounded evolution
uses substantially fewer candidates and validation rollouts than random
mutation while achieving higher pass rates and final performance.
These results suggest that the gains come from targeted retrieval-skill
refinement rather than broader random exploration.

\subsection{Detailed Behavior Analysis}

\input{table/router_skill_behavior}

\paragraph{Router and skill behavior.}
Table~\ref{tab:router-skill-behavior} shows that query-adaptive routing outperforms all fixed single-skill policies, with selections distributed across multiple skills.
This suggests that ERSkill learns a compact set of complementary retrieval behaviors and routes queries according to their differing information demands.
% See Appendix~\ref{app:analysis_of_route_and_skill} for deeper analysis.

\begin{figure*}[t]
    \centering
    \includegraphics[width=0.9\textwidth]{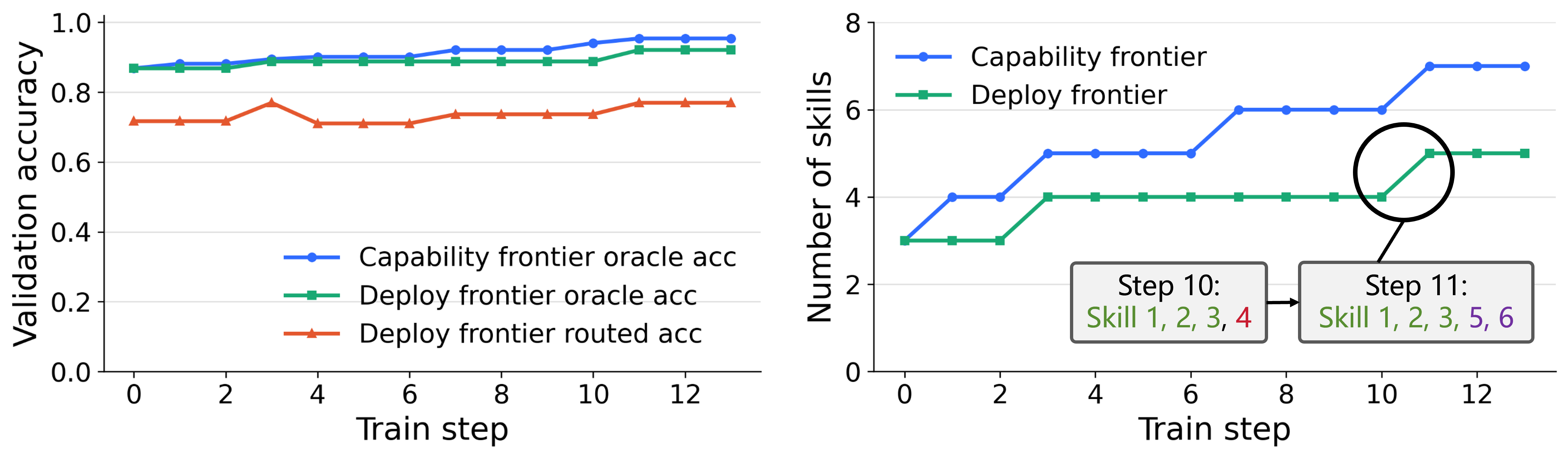}
    \caption{
    Evolution behavior.
    \textit{Left}: Oracle-side capability increases across both frontiers,
    while routed accuracy may temporarily drop and recover after frontier updates.
    \textit{Right}: Frontier sizes remain controlled by retaining
    non-redundant skills and pruning covered alternatives.
    }
    \label{fig:case_study}
\end{figure*}

\paragraph{Evolution behavior.}
Figure~\ref{fig:case_study} shows a representative LoCoMo run with
GPT-5.4-nano.
Capability- and deploy-frontier oracle accuracy increases steadily,
while routed accuracy may temporarily decrease after frontier updates
and recover as the router adapts.
Frontier sizes remain controlled by retaining skills with
non-redundant oracle value and pruning alternatives.

\input{table/different_judge}

\paragraph{Evaluator robustness.}
We re-evaluate the LoCoMo outputs with two independent LLM judges and
human annotations.
Table~\ref{tab:evaluator-robustness} shows high agreement with
GPT-4o-mini, with Cohen's $\kappa$ above 88\%, indicating that the
semantic-correctness results are robust to evaluator choice.
See Appendix~\ref{app:evaluator-robustness} for details.

\paragraph{Effect of train batch size.}
Train batch size trades off rollout evidence per update against evolution frequency.
Figure~\ref{fig:hyperparameter_batch_size} shows that a moderate batch size performs best, with 20 achieving the strongest overall results in our setting.
See Appendix~\ref{app:hyperparameter} for details.

\paragraph{Stability across repeated runs.}
Across five independent runs on both LLM backbones and all three benchmarks, ERSkill shows low run-to-run variation, with all coefficients of variation below 6.4\%.
See Appendix~\ref{app:stability} for details.

%% file: table/main_comparision.tex
\begin{table*}[t]
\caption{
Main comparison results.
We report F1, BLEU-1 (B1), and LLM-as-a-Judge (L-J) scores as percentages;
higher is better.
\textbf{ERSkill achieves the best overall average under both LLM backbones},
outperforming non-evolving, adaptive-retrieval, and evolving baselines.
It also remains strongest on LongMemEval when directly transferred from
LoCoMo without additional training.
$^*$ denotes this transfer setting.
}
\centering
% \scriptsize
\setlength{\tabcolsep}{2.0pt}
\resizebox{\textwidth}{!}{%
\fontsize{8pt}{10pt}\selectfont % 字号，行距
\begin{tabular}{lcccccccccccc}
\toprule
\multirow{2}{*}{\textbf{Methods}} & \multicolumn{3}{c}{\textbf{LoCoMo}} & \multicolumn{3}{c}{\textbf{LongMemEval$^*$}} & \multicolumn{3}{c}{\textbf{PerLTQA}} & \multicolumn{3}{c}{\textbf{Average}} \\
\cmidrule(lr){2-4}\cmidrule(lr){5-7}\cmidrule(lr){8-10}\cmidrule(lr){11-13}
& \textbf{F1} & \textbf{B1} & \textbf{L-J} & \textbf{F1} & \textbf{B1} & \textbf{L-J} & \textbf{F1} & \textbf{B1} & \textbf{L-J} & \textbf{F1} & \textbf{B1} & \textbf{L-J} \\
\midrule

\rowcolor{qwenbar}\multicolumn{13}{c}{\textbf{Qwen3-Next-80B-A3B-Instruct}} \\
\midrule
\multicolumn{13}{l}{\emph{Non-Evolving}} \\
A-Mem & 32.96 & 37.20 & 43.63 & 15.13 & 11.10 & 48.73 & 42.37 & 37.94 & 42.65 & 30.15 & 28.75 & 45.00  \\
MemoryOS & 34.03 & 37.43 & 46.49 & 18.15 & 12.60 & 49.74 & \underline{43.70} & 36.45 & 46.79 & 31.96 & 28.83 & 47.67 \\
LightMem & 42.63 & 42.89 & \underline{64.96} & 16.48 & 11.00 & \underline{55.83} & 43.14 & 35.30 & 39.54 & 34.08 & 29.73 & \underline{53.44} \\
\multicolumn{13}{l}{\emph{Adaptive Retrieval}} \\
Adaptive-RAG
& 22.72 & 22.88 & 31.21
& \underline{37.37} & \underline{41.91} & 49.74
& 42.44 & \underline{41.51} & \underline{49.48}
& 34.18 & 35.43 & 43.48 \\
\multicolumn{13}{l}{\emph{Evolving}} \\
Dynamic Cheatsheet & 34.96 & 33.22 & 56.05 & 33.59 & 34.13 & 51.26 & 35.50 & 40.88 & 45.96 & \underline{34.68} & \underline{36.08} & 51.09 \\
ReasoningBank & 15.77 & 9.12 & 55.41 & 17.02 & 12.85 & 54.31 & 43.05 & 37.55 & 48.44 & 25.28 & 19.84 & 52.72 \\
GEPA & 10.35 & 4.58 & 57.00 & 11.55 & 7.15 & 52.28 & 42.64 & 37.36 & 47.41 & 21.51 & 16.36 & 52.23 \\
MemSkill & \underline{43.25} & \underline{44.01} & 60.19 & 11.35 & 6.93 & 49.74 & 40.97 & 32.93 & 34.78 & 31.86 & 27.96 & 48.24   \\
\rowcolor{my_approach}
\textbf{ERSkill} & \textbf{49.96} & \textbf{50.18} & \textbf{70.06} & \textbf{46.79} & \textbf{53.31} & \textbf{59.39} & \textbf{51.89} & \textbf{44.07} & \textbf{54.45} & \textbf{49.55} & \textbf{49.19} & \textbf{61.30} \\

\midrule
\rowcolor{gptbar}\multicolumn{13}{c}{\textbf{GPT-5.4-nano}} \\
\midrule
\multicolumn{13}{l}{\emph{Non-Evolving}} \\
A-Mem & 30.40 & 29.33 & 49.04 & 17.07 & 13.77 & 42.13 & 36.06 & 32.95 & 42.65 & 27.84 & 25.35 & 44.61  \\
MemoryOS & 33.46 & 32.72 & \underline{59.55} & 15.50 & 10.57 & 45.17 & \underline{43.88} & 37.27 & 47.20 & \underline{30.95} & 26.85 & 50.64  \\
LightMem & \underline{36.43} & \textbf{37.63} & 58.28 & 15.95 & 10.73 & \underline{50.25} & 39.91 & 33.26 & 37.88 & 30.76 & 27.21 & 48.80  \\
\multicolumn{13}{l}{\emph{Adaptive Retrieval}} \\
Adaptive-RAG
& 27.30 & 28.47 & 37.26
& \underline{24.74} & \underline{28.95} & 35.02
& 30.21 & \underline{38.28} & 44.30
& 27.42 & \underline{31.90} & 38.86 \\
\multicolumn{13}{l}{\emph{Evolving}} \\
Dynamic Cheatsheet & 18.33 & 12.57 & \underline{59.55} & 21.19 & 18.31 & 43.14 & 43.44 & 36.46 & 50.72 & 27.65 & 22.45 & 51.14  \\
ReasoningBank & 17.65 & 11.88 & 57.96 & 18.88 & 13.34 & 49.74 & 43.14 & 36.16 & 48.44 & 26.56 & 20.46 & 52.05 \\
GEPA & 18.80 & 12.31 & \underline{59.55} & 21.74 & 18.58 & \underline{50.25} & 40.41 & 36.15 & \underline{51.34} & 26.98 & 22.35 & \underline{53.71} \\
MemSkill & 25.39 & 24.80 & 58.59 & 18.17 & 13.05 & 46.19 & 42.23 & 35.57 & 47.82 & 28.60 & 24.47 & 50.87   \\
\rowcolor{my_approach}
\textbf{ERSkill} & \textbf{38.68} & \underline{33.46} & \textbf{72.93} & \textbf{38.89} & \textbf{41.01} & \textbf{56.35} & \textbf{45.35} & \textbf{38.40} & \textbf{54.04} & \textbf{40.97} & \textbf{37.62} & \textbf{61.11}  \\

\bottomrule
\end{tabular}%
}
\label{tab:main-results}\vspace{-0.5em}
\end{table*}

%% file: table/comparison_shared_memory_small.tex
\begin{table}[t]
\centering
\caption{
Controlled comparison under a shared memory substrate.
For Dynamic Cheatsheet, ReasoningBank, and GEPA, ``Ours'' replaces
the original RAG memory with the same multi-primitive substrate used
by ERSkill while retaining each method's original evolution procedure.
We report the average L-J score across LoCoMo, LongMemEval, and PerLTQA;
$\Delta$ denotes the change from the original memory configuration.
The shared substrate yields mixed effects across methods and backbones,
while \textbf{ERSkill remains substantially stronger, indicating that its gains cannot be attributed to the richer memory substrate alone}.
Detailed results are in Appendix~\ref{app:same-memory}.
}
\setlength{\tabcolsep}{14pt}
\resizebox{\columnwidth}{!}{%
\small
\begin{tabular}{lcccccc}
\toprule
\multirow{2}{*}{\textbf{Method}}
& \multicolumn{3}{c}{\textbf{Qwen3-Next-80B-A3B-Instruct}}
& \multicolumn{3}{c}{\textbf{GPT-5.4-nano}} \\
\cmidrule(lr){2-4}\cmidrule(lr){5-7}
& \textbf{Orig.} & \textbf{Ours} & $\boldsymbol{\Delta}$
& \textbf{Orig.} & \textbf{Ours} & $\boldsymbol{\Delta}$ \\
\midrule

Dynamic Cheatsheet
& 51.09 & 48.91 & -2.18
& 51.14 & 53.32 & +2.18 \\

ReasoningBank
& 52.72 & 48.08 & -4.64
& 52.05 & 49.00 & -3.05 \\

GEPA
& 52.23 & 53.20 & +0.97
& 53.71 & 55.83 & +2.12 \\

\midrule
\rowcolor{my_approach}
\textbf{ERSkill}
& \multicolumn{3}{c}{\textbf{61.30}}
& \multicolumn{3}{c}{\textbf{61.11}} \\

\bottomrule
\end{tabular}%
}
\label{tab:same-memory}
\vspace{-0.5em}
\end{table}

%% file: table/cost_performance.tex
% Darker cells indicate higher cost.
\definecolor{costred}{RGB}{222,30,30}
\newcommand{\costcell}[2]{\cellcolor{costred!#1}\strut #2}

\begin{table*}[t]
\centering
\caption{
LLM-token cost on LoCoMo with GPT-5.4-nano.
``+ Shared substrate'' replaces the original RAG memory with ERSkill's
multi-primitive substrate.
Darker cells indicate higher cost;
``--'' denotes a non-applicable stage.
\textbf{ERSkill achieves the highest L-J score with low construction and
inference costs and moderate evolution cost}.
}
\setlength{\tabcolsep}{10.0pt}
\small

\resizebox{\linewidth}{!}{
\begin{tabular}{lcccc}
\toprule
\textbf{Method}
& \textbf{L-J}
& \shortstack{\textbf{Memory Construction}\\\textbf{(M Tokens / History Sample)}}
& \shortstack{\textbf{Evolution}\\\textbf{(M Tokens)}}
& \shortstack{\textbf{Inference}\\\textbf{(k Tokens / QA)}} \\
\midrule

\multicolumn{5}{l}{\emph{Non-Evolving}} \\

A-Mem
& 49.04
& \costcell{35}{1.177}
& --
& \costcell{16}{4.596} \\

MemoryOS
& 59.55
& \costcell{20}{0.357}
& --
& \costcell{18}{5.450} \\

LightMem
& 58.28
& \costcell{9}{0.154}
& --
& \costcell{12}{2.622} \\

\midrule
\multicolumn{5}{l}{\emph{Adaptive Retrieval}} \\

Adaptive-RAG
& 37.26
& --
& \costcell{27}{6.551}
& \costcell{16}{4.585} \\

\midrule
\multicolumn{5}{l}{\emph{Evolving}} \\

Dynamic Cheatsheet
& 59.55
& --
& \costcell{7}{0.306}
& \costcell{22}{9.077} \\

\quad + Shared substrate
& \underline{64.96}
& \costcell{5}{0.108}
& \costcell{13}{0.844}
& \costcell{35}{44.444} \\[1pt]

ReasoningBank
& 57.96
& --
& \costcell{5}{0.231}
& \costcell{5}{1.192} \\

\quad + Shared substrate
& 64.01
& \costcell{5}{0.108}
& \costcell{26}{6.114}
& \costcell{14}{3.532} \\[1pt]

GEPA
& 59.55
& --
& \costcell{29}{9.963}
& \costcell{16}{4.497} \\

\quad + Shared substrate
& 64.01
& \costcell{5}{0.108}
& \costcell{35}{24.256}
& \costcell{21}{7.867} \\[1pt]

MemSkill
& 58.59
& \costcell{10}{0.163}
& \costcell{34}{20.368}
& \costcell{6}{1.392} \\

\midrule
\textbf{ERSkill}
& \textbf{72.93}
& \costcell{5}{0.108}
& \costcell{29}{9.488}
& \costcell{9}{1.946} \\

\bottomrule
\end{tabular}
}

\label{tab:end-to-end-cost}
\vspace{-0.5em}
\end{table*}

%% file: table/ablation.tex
\begin{table*}[t]
\centering
\caption{
Ablation study.
Qwen and GPT denote Qwen3-Next-80B-A3B-Instruct and GPT-5.4-nano,
respectively.
\textbf{All components contribute to performance}, with skill evolution and
query-adaptive routing having the largest impact.
}
\setlength{\tabcolsep}{10.0pt}
\resizebox{\textwidth}{!}{
\begin{tabular}{llccc|ccc}
\toprule
\multirow{2}{*}{\textbf{Backbone}} 
& \multirow{2}{*}{\textbf{Variant}}
& \multicolumn{3}{c|}{\textbf{LoCoMo}}
& \multicolumn{3}{c}{\textbf{PerLTQA}} \\
\cmidrule(lr){3-5} \cmidrule(lr){6-8}
& & \textbf{F1} & \textbf{B1} & \textbf{L-J}
& \textbf{F1} & \textbf{B1} & \textbf{L-J} \\
\midrule

\multirow{6}{*}{Qwen}
& Full
& 49.96 & 50.18 & 70.06
& 51.89 & 44.07 & 54.45 \\

& w/o Skill Evolution
& 46.52 & 47.75 & 63.38
& 49.14 & 41.52 & 50.10 \\

& w/ Random Skill Mutation
& 42.67 & 43.32 & 58.60
& 47.00 & 40.09 & 51.13 \\

& w/o Router
& 43.46 & 44.17 & 58.28
& 48.50 & 40.45 & 49.27 \\

& w/o Double Frontier
& 44.95 & 44.81 & 60.51
& 49.69 & 43.38 & 52.58 \\

& w/o Experience Trie
& 47.74 & 48.56 & 67.20
& 51.46 & 43.16 & 53.20 \\

\midrule

\multirow{6}{*}{GPT}
& Full
& 38.68 & 33.46 & 72.93
& 45.35 & 38.40 & 54.04 \\

& w/o Skill Evolution
& 34.01 & 30.69 & 61.78
& 43.21 & 36.78 & 50.31 \\

& w/ Random Skill Mutation
& 33.47 & 29.75 & 62.42
& 42.53 & 36.37 & 50.72 \\

& w/o Router
& 32.42 & 28.76 & 60.19
& 41.88 & 34.06 & 48.86 \\

& w/o Double Frontier
& 35.25 & 31.17 & 64.97
& 44.63 & 37.80 & 52.38 \\

& w/o Experience Trie
& 36.35 & 32.31 & 67.51
& 44.87 & 38.61 & 53.83 \\

\bottomrule
\end{tabular}
}
\label{tab:ablation}
\end{table*}

%% file: table/retrieval_evidence_analysis.tex
\begin{table*}[t]
\centering
\small

% ==================== Left Table ====================
\begin{minipage}[t]{0.53\textwidth}
\centering

\captionof{table}{
Retrieval-level evidence analysis on LoCoMo with GPT-5.4-nano.
``w/o Skill Evolution'' retains the three initial skills and trains
the same router with a fixed skill set.
\textit{Evidence Hit Rate}, \textit{Overall Evidence Recall}, and
\textit{Full Evidence Recall Rate} measure whether at least one,
the aggregate fraction, and all gold evidence items are retrieved,
respectively.
ERSkill's skill evolution improves both evidence coverage and
answer quality.
All scores and rates are percentages.
}
\label{tab:evidence-quality}

\vspace{0.3em}
\setlength{\tabcolsep}{6.0pt}

\resizebox{\linewidth}{!}{
\begin{tabular}{lcc}
\toprule
\textbf{Metric}
& \textbf{w/o Skill Evolution}
& \textbf{Full} \\
\midrule
Evidence Hit Rate
& 66.99 & \textbf{80.13} \\
Overall Evidence Recall
& 57.20 & \textbf{70.58} \\
Full Evidence Recall Rate
& 50.64 & \textbf{62.18} \\
\midrule
F1
& 34.01 & \textbf{38.68} \\
B1
& 30.69 & \textbf{33.46} \\
L-J
& 61.78 & \textbf{72.93} \\
\bottomrule
\end{tabular}
}

\end{minipage}
\hfill
% ==================== Right Table ====================
\begin{minipage}[t]{0.45\textwidth}
\centering

\captionof{table}{
Comparison between ERSkill's trace-grounded evolution and
Random Skill Mutation on LoCoMo with
Qwen3-Next-80B-A3B-Instruct, using the same primitive library
and number of evolution batches.
Pass rates denote the percentages of candidates surviving each
validation stage.
Trace-grounded evolution achieves higher performance and pass
rates with substantially fewer candidates and validation rollouts.
}
\label{tab:evolution-strategy}

\vspace{0.3em}
\setlength{\tabcolsep}{6.0pt}

\resizebox{\linewidth}{!}{
\begin{tabular}{lcc}
\toprule
\textbf{Metric}
& \textbf{\shortstack{w/ Random\\Skill Mutation}}
& \textbf{Full} \\
\midrule
Evolution Batches
& 13 & 13 \\
Generated Candidates
& 30 & \textbf{8} \\
Validation Rollouts
& 4,560 & \textbf{1,216} \\
Capability Pass Rate
& 26.7 & \textbf{62.5} \\
Deploy Pass Rate
& 13.3 & \textbf{37.5} \\
L-J
& 58.60 & \textbf{70.06} \\
\bottomrule
\end{tabular}
}

\end{minipage}

\vspace{-0.5em}
\end{table*}

%% file: table/router_skill_behavior.tex
\begin{table*}[t]
\centering
\small

\caption{
Router and retrieval-skill behavior on LoCoMo with GPT-5.4-nano.
\textit{(a)} Query-adaptive routing versus fixed single-skill policies;
selection frequency denotes how often each skill is chosen.
\textit{(b)} Primitive usage among routed executions.
\textbf{ERSkill's adaptive routing outperforms every fixed policy} while distributing
queries across multiple skills, indicating complementary retrieval
behaviors and selective primitive use.
}

\begin{minipage}[t]{0.65\textwidth}
\centering
\subcaption{Router behavior.}
\setlength{\tabcolsep}{12pt}
\resizebox{\linewidth}{!}{
\begin{tabular}{lcccc}
\toprule
\textbf{Retrieval Policy}
& \textbf{Select. Freq.}
& \textbf{F1}
& \textbf{B1}
& \textbf{L-J} \\
\midrule
\textbf{ERSkill (router)}
& -- & \textbf{38.68} & \textbf{33.46} & \textbf{72.93} \\
Always Skill 1
& 55.73 & 37.90 & 32.77 & 70.56 \\
Always Skill 2
& 19.11 & 38.44 & 33.00 & 68.15 \\
Always Skill 3
& 19.11 & 31.78 & 28.13 & 62.42 \\
Always Skill 4
& 6.05 & 33.02 & 30.80 & 57.64 \\
\bottomrule
\end{tabular}
}
\end{minipage}
\hfill
\begin{minipage}[t]{0.33\textwidth}
\centering
\subcaption{Primitive usage.}
\setlength{\tabcolsep}{10pt}
\resizebox{\linewidth}{!}{
\begin{tabular}{lc}
\toprule
\textbf{Primitive} & \textbf{Usage Freq.} \\
\midrule
entity\_search  & 24.93 \\
lexical\_search & 52.01 \\
llm\_process    & 23.06 \\
\bottomrule
\end{tabular}
}
\end{minipage}

\label{tab:router-skill-behavior}
\vspace{-0.5em}
\end{table*}

%% file: table/different_judge.tex
% \begin{table}[t]
% \centering
% \caption{
% Evaluator consistency on LoCoMo with GPT-5.4-nano.
% Two independent LLM judges evaluate the full test set, while human evaluation
% uses 50 randomly selected samples.
% Agreement denotes identical binary decisions and Cohen's $\kappa$ is scaled
% to 0--100.
% All comparisons show high agreement, with $\kappa>88$, supporting the \textbf{robustness of our reported semantic-correctness results to evaluator choice}.
% }
% \setlength{\tabcolsep}{10.0pt}
% % \small
% \resizebox{0.7\linewidth}{!}{
% \begin{tabular}{lcc}
% \toprule
% \textbf{Evaluator Comparison}
% & \textbf{Agreement}
% & \textbf{Cohen's $\boldsymbol{\kappa}$} \\
% \midrule
% GPT-4o-mini vs. DeepSeek-V4-Flash
% & 95.54 & 88.70 \\
% GPT-4o-mini vs. Gemini-3.5-Flash-Lite
% & 95.54 & 88.79 \\
% GPT-4o-mini vs. Human
% & 98.00 & 94.93 \\
% \bottomrule
% \end{tabular}
% }
% \label{tab:evaluator-robustness}
% \vspace{-0.5em}
% \end{table}

\begin{figure*}[t]
\centering

% ==================== Left: Table ====================
\begin{minipage}[t]{0.71\textwidth}
    \vspace{0pt}
    \centering

    \captionof{table}{
        Evaluator consistency on LoCoMo with GPT-5.4-nano.
        Two independent LLM judges evaluate the full test set, while human evaluation
        uses 50 randomly selected samples.
        Agreement denotes identical binary decisions and Cohen's $\kappa$ is scaled
        to 0--100.
        All comparisons show high agreement, with $\kappa>88$, supporting the
        \textbf{robustness of our reported semantic-correctness results to evaluator choice}.
    }
    \label{tab:evaluator-robustness}

    \setlength{\tabcolsep}{10.0pt}
    \resizebox{\linewidth}{!}{
    \begin{tabular}{lcc}
        \toprule
        \textbf{Evaluator Comparison}
        & \textbf{Agreement}
        & \textbf{Cohen's $\boldsymbol{\kappa}$} \\
        \midrule
        GPT-4o-mini vs. DeepSeek-V4-Flash
        & 95.54 & 88.70 \\
        GPT-4o-mini vs. Gemini-3.5-Flash-Lite
        & 95.54 & 88.79 \\
        GPT-4o-mini vs. Human
        & 98.00 & 94.93 \\
        \bottomrule
    \end{tabular}
    }
\end{minipage}
\hfill
% ==================== Right: Figure ====================
\begin{minipage}[t]{0.27\textwidth}
    \vspace{0pt}
    \centering

    \includegraphics[width=\linewidth]{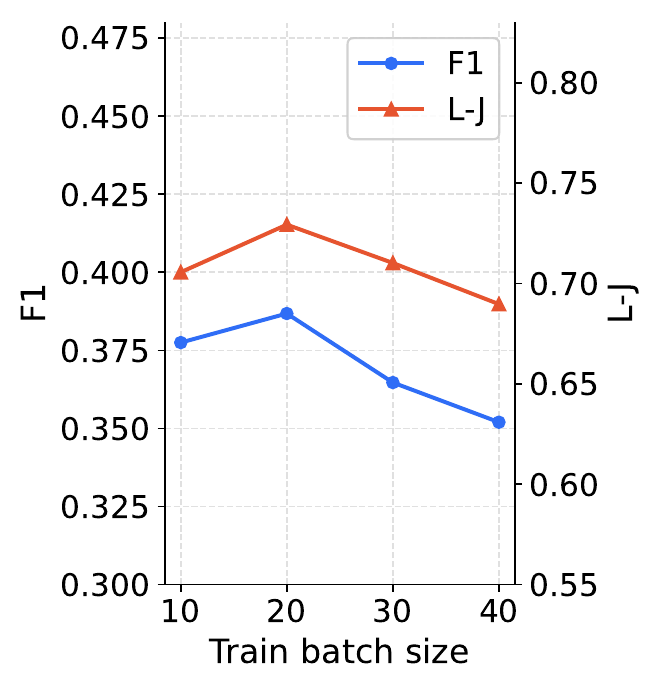}
    \captionof{figure}{
        Hyperparameter study on train batch size (LoCoMo, GPT-5.4-nano).
    }
    \label{fig:hyperparameter_batch_size}
\end{minipage}

% \vspace{-3.0em}
\end{figure*}

%% file: section/related.tex
\section{Related Work}

\paragraph{Agent Memory.}
Agent memory externalizes long interaction histories into persistent
stores for subsequent retrieval and reasoning~\citep{chhikara2025mem0,
packer2023memgpt, xu2025mem, fang2025lightmem, kang2025memory}.
Existing work improves memory construction and management through
pipeline design~\citep{xu2025mem, fang2025lightmem} or
training~\citep{yan2025memory, wang2025mem, yu2025memagent}, while
query-time retrieval is typically predefined.
We instead study adaptive memory retrieval, where memory access and
evidence construction adapt to query demands.

\paragraph{Agent Evolution.}
Recent work on evolving agents studies how agents can improve from interaction experience by using LLMs to analyze past task traces~\citep{zhao2024expel, ouyang2025reasoningbank, agrawal2025gepa, zhang2026memskill, zheng2025skillweaver, jiang2026xskill}.
A typical pipeline is to collect trajectories from previous tasks, let the LLM identify key factors, and distill them into reusable experiences, rules, prompts, or memory items that guide future behavior.
For example, ReasoningBank~\citep{ouyang2025reasoningbank} accumulates historical reasoning paths and asks the LLM to analyze correct and erroneous reasoning steps, turning the resulting insights into memory items for later reasoning.
MemSkill~\citep{zhang2026memskill} studies memory construction skills: it adjusts how an agent extracts memory items and improves these skills using task signals from downstream QA traces.
By contrast, ERSkill focuses on constructing skills for memory retrieval rather than memory construction.
It treats query-time memory access as an evolvable retrieval behavior and introduces an experience trie and a double-frontier mechanism to support
retrieval-skill evolution.

%% file: section/conclusion.tex
\section{Conclusion}

% We proposed \textbf{ERSkill}, a retrieval-centric framework for evolving, skill-guided agent memory retrieval.
% ERSkill represents memory access as executable retrieval skills composed from a shared primitive library, and uses a trained router to select the skill matching each query's information demand.
% To build effective and deployable skills, ERSkill co-evolves the skill set and router with an experience trie and a double-frontier mechanism, separating capability expansion from router-facing deployment.
% Experiments on three agent memory benchmarks show that ERSkill outperforms strong non-evolving and evolving baselines across different backbone LLMs, while further analyses demonstrate cross-dataset transfer, favorable cost-performance trade-offs, and stable evolution.
% These results highlight retrieval-side evolution as a promising direction for long-term LLM agents.

We proposed ERSkill, a retrieval-centric framework that evolves
executable memory-retrieval skills and adaptively routes queries among
them.
Its experience trie and double-frontier design support skill-router
co-evolution while separating capability expansion from router-facing
deployment.
Across three benchmarks and two LLM backbones, ERSkill outperforms
strong baselines, transfers across datasets, and shows favorable
LLM-token/quality trade-offs, highlighting retrieval-side evolution as
a promising direction for long-term LLM agents.

%% file: section/appendix.tex
\appendix

\section{Details for Memory Storage}
\label{appen:details_for_memory_storage}

\subsection{Structured Memory Construction}
\label{appen:structured_memory}

We describe the construction procedure of the structured memory used by ERSkill. Given an interaction history $D$, ERSkill compiles it into a retrieval-first memory store:
\begin{equation}
    M(D) = (\mathcal{A}, \mathcal{I}, \mathcal{G}),
    \label{eq:structured-memory-store}
\end{equation}
where $\mathcal{A}$ denotes a set of atom-level memory records, $\mathcal{I}$ denotes a collection of indexes that provide entry points for search primitives, and $\mathcal{G}$ denotes a collection of graphs used by expansion primitives. The construction pipeline is summarized as:
\begin{equation}
\begin{aligned}
    \text{sample}
    &\rightarrow \text{atoms}
    \rightarrow \text{atom/entity embeddings} \\
    &\rightarrow \text{similarity edges}
    \rightarrow \text{relation edges}
    \rightarrow \text{runtime indexes}.
\end{aligned}
\label{eq:memory-construction-pipeline}
\end{equation}

\paragraph{Memory atoms.}
We first convert each raw benchmark sample into a sequence of text atoms. The segmentation strategy depends on the dataset. For LoCoMo, we read the session fields from the conversation together with the corresponding session timestamps, and use turn-pairs as the atom granularity. LongMemEval is processed in a similar way, using its historical sessions, session dates, and session identifiers. For PerLTQA, we first flatten heterogeneous memory sources, including profiles, relationships, events, and dialogues, into memory records. For profiles, relationships, and events, each individual record is directly treated as one text atom. For dialogues, each session is treated as one text atom.

Each text atom is then wrapped as a memory atom. We represent each memory atom as
\begin{equation}
    a_i = (\texttt{atom\_id}, \texttt{text}, \texttt{timestamp}, \texttt{entity\_set}).
    \label{eq:memory-atom}
\end{equation}
We further enrich each atom with several access signals. First, we extract the entity set from the atom text using named entity recognition implemented by the spaCy package~\citep{Honnibal_spaCy_Industrial-strength_Natural_2020}. Second, for LoCoMo and LongMemEval, we preserve the corresponding timestamp as temporal metadata. For PerLTQA, the timestamp field is left empty when no comparable session-level timestamp is available.

\paragraph{Embedding indexes.}
After constructing atoms, we embed both atoms and entities. For atom embeddings, we collect all atom texts and encode them with the context-side retriever encoder:
\begin{equation}
    \mathbf{h}_i = \mathrm{Embed}_{\mathrm{ctx}}(a_i.\texttt{text}).
    \label{eq:atom-embedding}
\end{equation}
The resulting vectors are stored as atom embeddings and are used by dense retrieval. We also collect all unique entity strings from the atom entity sets:
\begin{equation}
    \mathcal{E} = \bigcup_{a_i \in \mathcal{A}} a_i.\texttt{entity\_set},
    \label{eq:entity-set}
\end{equation}
encode them with the same embedding interface, and store the resulting entity texts and entity embeddings. These entity embeddings support entity-centered search and graph activation.

\paragraph{Similarity graph.}
We construct a similarity graph over memory atoms using atom embeddings. For each atom, we compute cosine similarity to all other atoms, remove the self-edge, keep the top-$k$ ($k=5$ in our implementation) most similar atoms, and discard non-positive similarities. Each retained neighbor produces a directed edge:
\begin{equation}
    e_{i,j}^{\mathrm{sim}}
    =
    (a_i, a_j, \texttt{similarity}, s_{i,j}, \texttt{embedding\_topk}),
    \label{eq:similarity-edge}
\end{equation}
where $s_{i,j}$ is the cosine similarity score:
\begin{equation}
    s_{i,j}
    =
    \frac{\mathbf{h}_i^\top \mathbf{h}_j}
    {\|\mathbf{h}_i\|_2 \|\mathbf{h}_j\|_2}.
    \label{eq:cosine-similarity}
\end{equation}
This graph is used by \texttt{similarity\_expand} and also provides candidate neighbors for relation extraction.

\paragraph{Relation graph.}
We further construct typed relation edges between atoms. The relation graph contains both rule-based temporal-neighbor edges and LLM-extracted semantic relation edges. For temporal-neighbor edges, we group atoms by session, sort them by turn order, and connect adjacent atoms when their time keys overlap:
\begin{equation}
    a_i \leftrightarrow a_j
    \quad
    \text{if}
    \ 
    a_i.\texttt{time\_keys} \cap a_j.\texttt{time\_keys} \neq \varnothing
    \quad
    \text{and}
    \ 
    a_i,a_j \text{ are adjacent in the same session}.
    \label{eq:temporal-neighbor-rule}
\end{equation}
These edges are added bidirectionally with relation type \texttt{TemporalNeighbor}. This rule is applied to session-based datasets such as LoCoMo and LongMemEval.

For semantic relation edges, we start from the similarity neighbors of each source atom and keep the top candidates. For non-PerLTQA datasets, we only retain candidate neighbors that are temporally later than the source atom, so that extracted relations follow the forward temporal direction of the interaction history (bi-directional for PerLTQA). The source atom and its candidate neighbors are then given to an LLM relation extractor. The extractor is constrained to output one of the following labels:
\begin{equation}
\begin{aligned}
    \mathcal{R} = \{
    \texttt{Changed},\ \texttt{Cause},\ \texttt{Reason},\ \texttt{HinderedBy},\ \texttt{React},\ \texttt{Want},\ \texttt{none}
    \}.
\end{aligned}
\label{eq:relation-label-set}
\end{equation}
The relation set is inspired by \citep{hwang2021comet}. If the predicted label is not \texttt{none}, we add a typed relation edge:
\begin{equation}
    e_{i,j}^{\mathrm{rel}}
    =
    (a_i, a_j, r_{i,j}, \max(0.5, s_{i,j})),
    \label{eq:llm-relation-edge}
\end{equation}
where $r_{i,j} \in \mathcal{R} \setminus \{\texttt{none}\}$ is the predicted relation label and $s_{i,j}$ is the similarity score of the candidate neighbor. The extraction prompts are provided in Appendix~\ref{appen:llm_relation_extraction_prompt}.

\paragraph{Runtime indexes.}
Finally, we finalize the memory store by constructing runtime indexes for efficient primitive execution. The finalized indexes include: an atom-id-to-atom map, token and inverse-document-frequency statistics for BM25-style lexical retrieval, an entity-to-atom inverted index, an atom-to-entity map, a time-key index, a relation adjacency list, and a similarity adjacency list. These indexes make the memory directly executable by retrieval primitives. The resulting memory store is serialized and cached, so later skill execution can reuse the same structured memory without reconstructing embeddings, graphs, or indexes.

% \paragraph{Primitive interface.}
% ERSkill exposes the structured memory through a fixed primitive library. The primitive library contains three types of operators. Search primitives enter memory through indexes in $\mathcal{I}$: \texttt{dense\_search} retrieves atoms by query--atom embedding similarity, \texttt{lexical\_search} retrieves atoms by BM25-style surface-form matching, and \texttt{entity\_search} anchors retrieval through entities mentioned or implied in the query. Expansion primitives grow the current evidence state through graphs in $\mathcal{G}$: \texttt{similarity\_expand} follows similarity edges, \texttt{relation\_expand} follows typed relation edges, and \texttt{temporal\_focus\_expand} expands around temporally relevant atoms or neighboring turns. In addition, \texttt{llm\_process} is used as a processing primitive for query rewriting, evidence filtering, state normalization, and deriving control signals for subsequent expansion.
% We summarize the primitives in Table~\ref{tab:primitive-library}. All primitives are provided to the LLM via Python function calls.

\paragraph{Primitive interface.}
ERSkill exposes the structured memory through a fixed primitive library. Each primitive is implemented as a Python function call and is executed by a shared primitive executor. Given the current query, the structured memory, and the current retrieval state, a primitive returns a set of candidate atoms together with execution signals, and the executor merges or replaces these candidates into the active evidence state according to the skill program.

The primitive library contains three types of operators. Search primitives enter the global memory through indexes in $\mathcal{I}$.
\texttt{dense\_search} embeds the query and retrieves atoms by cosine similarity against the cached atom embeddings.
\texttt{lexical\_search} performs BM25-style surface-form matching using the token and inverse-document-frequency statistics built during memory finalization.
\texttt{entity\_search} first extracts entity from the query, matches them against cached entity embeddings, activates the entity--atom index, and ranks candidate atoms with a personalized PageRank procedure that combines entity seed scores with query--atom semantic priors.

Expansion primitives grow the current evidence state from the active candidate atoms through graphs in $\mathcal{G}$. \texttt{similarity\_expand} follows precomputed similarity edges from the current active atoms and adds semantically neighboring atoms. \texttt{relation\_expand} follows typed relation edges from the current active atoms, optionally constrained by a preferred relation such as \texttt{Cause}, \texttt{Reason}, \texttt{Changed}, or \texttt{TemporalNeighbor}. \texttt{temporal\_focus\_expand} retrieves atoms whose timestamps fall within an explicit time range and optionally reranks them with query--atom embedding similarity; it therefore requires the skill program or an upstream processing step to provide a valid \texttt{time\_range}.

In addition, \texttt{llm\_process} serves as a processing primitive rather than a retrieval primitive. It reads the question, the current query, and the current evidence context, then writes structured variables back to the retrieval state, such as \texttt{preferred\_relations}, \texttt{time\_range}, a rewritten \texttt{current\_query}, or a \texttt{view\_summary}. These variables can subsequently control expansion primitives, for example by selecting the relation label for \texttt{relation\_expand} or deriving the time window for \texttt{temporal\_focus\_expand}. We summarize the primitives in Table~\ref{tab:primitive-library}. All primitives are provided to the LLM via function calls.

\begin{table}[ht]
\centering
\caption{Primitive library in ERSkill. Search primitives locate candidate atoms, expansion primitives extend evidence, and Process primitives serve for intermediate processing.}
\resizebox{\columnwidth}{!}{
\setlength{\tabcolsep}{2.0pt}
\fontsize{7pt}{8pt}\selectfont
\begin{tabular}{lll}
\hline
\textbf{Type} & \textbf{Primitive} & \textbf{Function} \\
\hline
Search &
\texttt{entity\_search} &
Anchor retrieval by entities mentioned or implied in the query. \\
Search &
\texttt{lexical\_search} &
Retrieve atoms with surface-form lexical matches. \\
Search &
\texttt{dense\_search} &
Retrieve atoms by semantic similarity to the query. \\
\hline
Expand &
\texttt{temporal\_focus\_expand} &
Expand around temporally relevant atoms or neighboring turns. \\
Expand &
\texttt{similarity\_expand} &
Expand to atoms connected by semantic similarity. \\
Expand &
\texttt{relation\_expand} &
Expand along typed relations between atoms. \\
\hline
Process &
\texttt{llm\_process} &
Rewrite queries, filter evidence, or normalize the retrieval state, etc. \\
\hline
\end{tabular}
}
\label{tab:primitive-library}
\end{table}

\section{Details for Evolution}
\label{appen:details_for_evolution}

\subsection{Algorithm of Evolution}
\label{appen:algo_evolution}

We summarize the procedure of the Skill-Router Co-Evolution in Algorithm~\ref{alg:skill_router_evolution}.

\begin{algorithm}[ht]
\caption{Skill-Router Co-Evolution in ERSkill}
\label{alg:skill_router_evolution}
\begin{algorithmic}[1]
\Require Training batches $\{\mathcal{Q}_t\}_{t=0}^{T-1}$, validation set $\mathcal{Q}_{\mathrm{val}}$, seed skill set $\mathcal{K}_{\mathrm{seed}}$, primitive library $\mathcal{P}$, initial router parameters $\theta_0$, routed-gain margin $\gamma_{\mathrm{route}}$, compactness tolerance $\xi_{\mathrm{drop}}$
\Ensure Final deploy frontier $\mathcal{B}_T$, router parameters $\theta_T$, and experience trie $\mathcal{T}_T$

\State Initialize $\mathcal{C}_0\gets\mathcal{K}_{\mathrm{seed}}$, $\mathcal{B}_0\gets\mathcal{K}_{\mathrm{seed}}$, and initialize $\mathcal{T}_0$ with the paths in $\mathcal{K}_{\mathrm{seed}}$
\State Initialize the router training window $\mathcal{W}_0\gets\emptyset$

\For{$t=0,\ldots,T-1$}
    \State Roll out each skill $\kappa\in\mathcal{C}_t$ on each query $q\in\mathcal{Q}_t$ to obtain query--skill scores, execution traces, and ability overlaps
    \State Write rollout results of $\mathcal{C}_t$ into $\mathcal{T}_t$
    
    \State Generate candidate skills $\mathcal{U}_t$ by editing paths from $\mathcal{C}_t$ and filtering out paths already recorded in $\mathcal{T}_t$
    \State Roll out each candidate $\kappa\in\mathcal{U}_t$ on $\mathcal{Q}_t$ and write its path and train-batch rollout results into $\mathcal{T}_t$
    
    \State Compute the train-side temporary frontier $\widetilde{\mathcal{C}}_{t+1}\gets\Phi(\mathcal{C}_t\cup\mathcal{U}_t;\mathcal{Q}_t)$
    \State Let $\mathcal{V}_t\gets\mathcal{U}_t\cap\widetilde{\mathcal{C}}_{t+1}$ be the candidates retained by train-side recomputation
    
    \State Roll out each retained candidate $\kappa\in\mathcal{V}_t$ on $\mathcal{Q}_{\mathrm{val}}$
    \State Update the capability frontier $\mathcal{C}_{t+1}\gets\Phi(\mathcal{C}_t\cup\mathcal{V}_t;\mathcal{Q}_{\mathrm{val}})$
    
    \State Update $\mathcal{W}_{t+1}$ with rollout instances $(q,\mathcal{K}_q,\{r(q,\kappa)\}_{\kappa\in\mathcal{K}_q})$ collected in the current step
    \State Continually update the router from $\theta_t$ to $\theta_{t+1}$ using mini-batches from $\mathcal{W}_{t+1}$
    
    \If{$\mathcal{C}_{t+1}\neq\mathcal{C}_t$}
        \State Derive the deploy-update candidate set $\mathcal{H}_t\gets\mathcal{U}_t\cap\mathcal{C}_{t+1}$
        \State Construct the candidate deploy frontier $\mathcal{B}'_t\gets\Phi(\mathcal{B}_t\cup\mathcal{H}_t;\mathcal{Q}_{\mathrm{val}})$
        \State Compute $\Delta_{\mathrm{route}}(\mathcal{B}'_t;\theta_{t+1})\gets\mathrm{Routed}(\mathcal{B}'_t,\theta_{t+1};\mathcal{Q}_{\mathrm{val}})-\mathrm{Routed}(\mathcal{B}_t,\theta_{t+1};\mathcal{Q}_{\mathrm{val}})$
        \If{$\Delta_{\mathrm{route}}(\mathcal{B}'_t;\theta_{t+1})\ge\gamma_{\mathrm{route}}$ \textbf{or} $(\Delta_{\mathrm{route}}(\mathcal{B}'_t;\theta_{t+1})\ge-\xi_{\mathrm{drop}}$ \textbf{and} $|\mathcal{B}'_t|\le|\mathcal{B}_t|)$}
            \State Accept the deploy update and set $\mathcal{B}_{t+1}\gets\mathcal{B}'_t$
        \Else
            \State Reject the deploy update and set $\mathcal{B}_{t+1}\gets\mathcal{B}_t$
        \EndIf
    \Else
        \State Set $\mathcal{B}_{t+1}\gets\mathcal{B}_t$
    \EndIf
    
    \State Write frontier status and deploy accept/reject outcomes into $\mathcal{T}_t$
    \State Set $\mathcal{T}_{t+1}\gets\mathcal{T}_t$
\EndFor

\State \Return $\mathcal{B}_T$, $\theta_T$, $\mathcal{T}_T$
\end{algorithmic}
\end{algorithm}

\subsection{Skill Candidate Generation}
\label{appen:skill_cand_gen}

We organize the skill evolution process as a three-stage agent workflow. First, we select a subset of rollouts in the train batch for skill candidate generation. For each selected trace (rollout), we invoke an analyzer to perform trace-grounded case analysis. Failed traces and successful traces are handled by two separate analyzer prompts: the failed-trace analyzer diagnoses where the current skill breaks, including the root cause, failure mode, and key nodes leading to the mismatch, while the success-trace analyzer explains why the skill matches the task and identifies the nodes that materially contribute to the successful execution. Next, the designer aggregates these case-level analyses together with the current skill catalog, frontier skills, primitive catalog, aggregate skill metrics, oracle performance, and recent evolution history. Based on this integrated diagnosis, it produces high-level evolution decisions, either keeping the skill set unchanged or proposing a new path candidate. Finally, the generator takes each approved designer proposal and materializes it into concrete skill fields, including the skill name, description, information preference, and executable program JSON, while ensuring that the generated skill remains within the proposal scope and is novel with respect to existing canonical programs.

\begin{tcolorbox}[promptbox, title={Analyzer Prompt for Failed Trace}]
\begin{Verbatim}[
  fontsize=\small,
  breaklines=true,
  breakanywhere=true,
  breaksymbolleft={},
  breaksymbolright={}
]
You are the analyzer for retrieval skill evolution.

Role:
Per-case analyzer for one failed run. Produce trace-grounded diagnosis only; do not propose skill-level actions.

Task:
Trace how evidence and state changed across the program, judge how well the current skill matches the task, judge whether the canonical program supports the declared skill, and assign exactly one diagnostic root cause.

Rules:
- Use full trace, retrieval_metrics, evidence_view, and final_state.
- Root causes:
  - skill_capability_gap: current description / information_preference / program is insufficient for the task, including cases where the needed evidence was not retrieved.
  - skill_over_broad_boundary: case appears attracted by the wrong skill; main issue is ownership or boundary clarity.
  - answer_generation_mismatch: useful evidence is present but the final answer is still weak.
- Also classify the failure mode as exactly one of:
  - retrieval_gap: needed evidence was not retrieved or the retrieval path is insufficient.
  - synthesis_gap: useful evidence is already present but extraction, reasoning, or answer generation still fails.
- Focus on where the current skill-task match is supported and where it breaks. Do not recommend skill changes.
- Use key_nodes to identify the nodes that lead to failure or mismatch with the intended behavior.
- Use overall_trace_judgment to summarize the skill's overall failure pattern.
- Use node_findings to explain each node's role, evidence or state change, and the observed mismatch, missing support, or drift.

Output Schema:
Return strict JSON with keys:
- case_index
- case_query
- case_groundtruth
- case_prediction
- case_chosen_skill
- root_cause
- failure_mode
- overall_trace_judgment
- program_support_status
- key_nodes
- node_findings
key_nodes must be a list of objects with keys:
- node_id
- primitive
- undesired_behavior
Each node_findings item must contain:
- node_id
- primitive
- role_in_program
- evidence_delta
- query_or_state_delta
- observed_problem

Case Input:
<case_payload>
\end{Verbatim}
\end{tcolorbox}

\begin{tcolorbox}[promptbox, title={Analyzer Prompt for Successful Trace}]
\begin{Verbatim}[
  fontsize=\small,
  breaklines=true,
  breakanywhere=true,
  breaksymbolleft={},
  breaksymbolright={}
]
You are the analyzer for retrieval skill evolution.

Role:
Per-case analyzer for one successful run. Produce structured success explanation only; do not invent a discrete success taxonomy.

Task:
Explain why the current skill matches the task in this case, which nodes materially contributed to success, and whether the canonical program supports the declared skill.

Rules:
- Use full trace, retrieval_metrics, evidence_view, and final_state.
- Keep the analysis narrative and structured; do not create a discrete success pattern label.
- Use key_nodes to identify the nodes that materially contributed to success.
- Use overall_trace_judgment to summarize the skill's overall success pattern.
- Focus on why the skill-task match works here; do not recommend skill changes.

Output Schema:
Return strict JSON with keys:
- case_index
- case_query
- case_groundtruth
- case_prediction
- case_chosen_skill
- overall_trace_judgment
- program_support_status
- key_nodes
- node_findings
key_nodes must be a list of objects with keys:
- node_id
- primitive
- contribution
Each node_findings item must contain:
- node_id
- primitive
- role_in_program
- evidence_delta
- query_or_state_delta

Case Input:
<case_payload>
\end{Verbatim}
\end{tcolorbox}

\begin{tcolorbox}[promptbox, title={Designer Prompt}]
\begin{Verbatim}[
  fontsize=\small,
  breaklines=true,
  breakanywhere=true,
  breaksymbolleft={},
  breaksymbolright={}
]
You are the designer for retrieval skill evolution.

Role:
Integrate trace-grounded case analyses into skill-set level diagnosis and produce up to <shadow_candidate_top_k> high-level executable proposals.

Task:
Use train-batch trace experience as the main diagnostic signal and capability-level aggregate signals as the global constraint. Combine them with the current skill catalog to form one integrated conclusion about coverage gaps and reusable successful patterns, then output high-level no_change / new_path_candidate proposals.

Analysis Views:
- Failed view: failed_trace_analyses are case-level and already contain per-case root_cause. For each failed case, analyze why all skills failed on that query.
- Success view: read success_trace_analyses from a skill-level view. Summarize what successful patterns for each skill look like, what boundary they anchor, and which node or evidence behaviors make the skill work.
- Aggregate capability view: use skill_metrics and oracle_acc to understand average skill strength, each skill's unique ownership, and the batch oracle ceiling.
- root_cause is a diagnostic signal only. Do not let any single root_cause label directly determine the proposal action.
- failure_mode is the primary execution-path signal. Treat retrieval_gap as evidence that retrieval path or graph support may be insufficient; treat synthesis_gap as evidence that useful evidence may already exist and answer extraction or processing may be the main issue.
- Proposals must come from the final integrated conclusion across failed cases, success patterns, skill definitions, canonical programs, aggregate capability signals, and recent history.
- Do not tunnel on one bad sample. Explicitly reconcile local failures with aggregate capability signals before proposing changes.
- Use no_change when the integrated conclusion does not support a new path. Use new_path_candidate when the current inventory still leaves a real capability gap and a new fixed retrieval path is justified.
- It is acceptable, and often beneficial, to recombine meaningful patterns from multiple existing skills when the integrated evidence suggests that a new path built from those patterns would better cover the missing capability.
- For retrieval_gap, it is acceptable to consider search / expand / graph rewrite / primitive swap changes when aggregate evidence supports retrieval-path redesign. For synthesis_gap, it is acceptable to consider llm_process or lighter answer-extraction changes when evidence is already present.
- Do not hard-code search or expand as the default action. Choose retrieval-path change only when the integrated evidence supports it.
- Use aggregate signals aggressively when deciding which existing patterns to borrow from first. Prefer borrowing from skills that are relatively lower-value within the current skill set when possible, and avoid disturbing relatively higher-value patterns unless the integrated evidence strongly justifies it.
- Treat avg_acc and unique_win_ratio as comparative signals within the current skill set, not absolute thresholds.
- Designer proposals must stay high-level. Do not specify exact node ids, graph layouts, sequence plans, or connection details.
- Every valid new_path_candidate represents a genuinely new fixed path skill, not a metadata-only or no-program-change edit.
- Same program under a different name is forbidden.
- Every new_path_candidate must be a modification of exactly one current capability-frontier skill.
- source_skill_name must be chosen from the current capability frontier only.
- Retired skills may appear in trie history for reference, but they cannot be chosen as source_skill_name.
- Prefer modifications that aggressively cover currently uncovered queries or create unique capability regions.
- Do not propose a mild patch whose only meaningful change is adding one llm_process node.

Output Schema:
Return strict JSON with keys:
- skill_set_diagnosis: object with keys coverage_gaps, reusable_patterns, candidate_new_skills
- proposals: list of objects with keys action, root_cause, problem_pattern, trigger_source, intended_program_novelty, rationale, expected_gain, reasoning, and only for new_path_candidate also new_skill_name and source_skill_name
Field constraints:
- action must be exactly one of: no_change, new_path_candidate
- trigger_source must be exactly: oracle_hard
- If action is new_path_candidate, new_skill_name must be a lowercase slug
- If action is new_path_candidate, source_skill_name must be a lowercase slug from the current capability frontier
- If action is no_change, do not provide new_skill_name
- Use short enum-like values for action and trigger_source, not sentences, booleans, or explanations

Inputs:
Primitive catalog:
<primitive_catalog>

Existing skills with canonical programs:
<skill_catalog>

Current capability frontier parent skills:
<frontier_skill_catalog>

Trie history summary:
<trie_history_summary>

Recent evolution history:
<recent_history>

Aggregate skill metrics:
<skill_metrics>

Batch oracle capability:
<oracle_acc>

Failed trace analyses:
<failed_trace_analyses>

Success trace analyses:
<success_trace_analyses>
\end{Verbatim}
\end{tcolorbox}

\begin{tcolorbox}[promptbox, title={Generator Prompt}]
\begin{Verbatim}[
  fontsize=\small,
  breaklines=true,
  breakanywhere=true,
  breaksymbolleft={},
  breaksymbolright={}
]
You are the generator for retrieval skill evolution.

Role:
Materialize one approved proposal into structured skill fields. Do not output markdown and do not re-judge the proposal.

Task:
Treat the selected proposal as the final decision. Generate name, description, information_preference, and program_json that stay within that proposal's scope, are derived from source_skill_name as the parent skill, and keep the resulting program novel relative to existing canonical programs.

Program Rules:
- Use only valid existing primitives; do not invent unsupported arguments.
- Do not reinterpret the proposal. Do not change action, ownership, or the fact that this is a new-path candidate.
- Keep action and new_skill_name consistent with the selected proposal.
- source_skill_name is the parent skill. Treat its canonical program as the base program to modify.
- Do not broaden the skill description beyond the proposal.
- Ensure the canonical program supports the claimed skill behavior.
- Prefer genuine graph edits when needed: prepend a node, insert a node anywhere, remove a node, swap a primitive, or rewrite control flow.
- Do not default to changing only an llm_process prompt if the real issue is retrieval graph design.
- When using expand-style primitives, usually place an upstream llm_process node to determine the control inputs before expansion. For example: choose relation labels before relation_expand, derive time_range before temporal_focus_expand, and identify which memory or evidence subset should be expanded before similarity_expand.
- temporal_focus_expand requires an explicit time_range input and is usually not a first-hop primitive by itself; prefer an upstream retrieval step plus llm_process to derive time_range before calling it.
- Every valid output must contain a genuinely new fixed path program, not a metadata-only or no-program-change rewrite.
- Do not produce a proposal whose only meaningful structural change is adding one llm_process node.

Few-shot Examples:
Example 1
Selected proposal:
{
  "action": "new_path_candidate",
  "new_skill_name": "semantic-surface-clue",
  "reasoning": "The current skill retrieves semantically related evidence but often misses exact answer-bearing wording. Keep the same owner but tighten it with a second retrieval step."
}
Good generator behavior:
- keep the proposal high-level intent
- materialize the graph as a real retrieval-path edit such as dense_search followed by lexical_search
- do not replace the change with only an llm_process prompt rewrite

Example 2
Selected proposal:
{
  "action": "new_path_candidate",
  "new_skill_name": "entity-relational-chain",
  "reasoning": "The current skill finds the right entity, but the evidence chain is too shallow and does not expand along the relation that actually links to the answer. Keep the same owner and extend the retrieval path."
}
Good generator behavior:
- keep the skill identity
- preserve the retrieval owner
- materialize the graph as an entity_search followed by relation_expand, or entity_search followed by relation_expand and lexical_search when the relation expansion still needs surface filtering
- do not replace the change with only an answer-side llm_process step

Example 3
Selected proposal:
{
  "action": "new_path_candidate",
  "new_skill_name": "controlled-temporal-expansion",
  "reasoning": "The skill needs a controlled expansion, but the expansion inputs are not directly available from the raw question. Add structure that first determines the control signal and then expands."
}
Good generator behavior:
- keep the proposal high-level and turn it into a controlled graph edit
- insert an llm_process only to derive control inputs such as relation labels or time_range
- then use that control output to drive relation_expand or temporal_focus_expand
- do not stop at llm_process as the final fix when the proposal requires controlled expansion

Serialization Rules:
- program_json must be a JSON object, not a string.
- It must serialize with json.dumps and parse with json.loads.
- Any prompt_template must be a single JSON string value.
- Use escaped JSON-safe strings; do not embed unsafe nested JSON examples.
- Prefer prose examples over literal embedded JSON when possible.

Required Output:
Return strict JSON with keys:
- action
- source_skill_name
- new_skill_name
- description
- information_preference
- program_json
- reasoning

Inputs:
Selected proposal:
<proposal>

Primitive catalog:
<primitive_catalog>

Existing skills:
<skill_catalog>

Current skill markdown:
<current_markdown>
\end{Verbatim}
\end{tcolorbox}

\section{Proof}
\label{appen:proof}

\begin{proposition}[Oracle-coverage monotonicity of the two-frontier update. Restatement of Proposition~\ref{prop:oracle-safe-frontier}]
\label{appen:prop:oracle-safe-frontier}
Denote the oracle coverage as $\mathrm{OCov}(\mathcal{K};\mathcal{Q})=\frac{1}{|\mathcal{Q}|}\sum_{q\in\mathcal{Q}}g_{\mathcal{K}}(q)$.
For every evolution step $t$, ERSkill satisfies $\mathrm{OCov}(\mathcal{C}_{t+1};\mathcal{Q}_{\mathrm{val}})\geq \mathrm{OCov}(\mathcal{C}_t;\mathcal{Q}_{\mathrm{val}})$ and $\mathrm{OCov}(\mathcal{B}_{t+1};\mathcal{Q}_{\mathrm{val}})\geq \mathrm{OCov}(\mathcal{B}_t;\mathcal{Q}_{\mathrm{val}})$.
\end{proposition}

\begin{proof}
Let $r(q,\kappa)\in[0,1]$ denote the validation score obtained by skill $\kappa$ on query $q$.
For any skill set $\mathcal{K}$, define its oracle score profile on $\mathcal{Q}_{\mathrm{val}}$ as
\begin{equation}
g_{\mathcal{K}}(q)
=
\max_{\kappa\in\mathcal{K}} r(q,\kappa),
\quad q\in\mathcal{Q}_{\mathrm{val}}.
\label{eq:oracle-score-profile}
\end{equation}
The oracle coverage score is the average oracle score over validation queries:
\begin{equation}
\mathrm{OCov}(\mathcal{K};\mathcal{Q}_{\mathrm{val}})
=
\frac{1}{|\mathcal{Q}_{\mathrm{val}}|}
\sum_{q\in\mathcal{Q}_{\mathrm{val}}}
g_{\mathcal{K}}(q).
\label{eq:ocov-score-definition}
\end{equation}

We first show that $\Phi$ preserves the oracle score profile on the batch used for recomputation.
Consider one pruning step with active set $\mathcal{K}_{\mathrm{act}}$.
A skill $\kappa$ is removed only if removing it does not reduce the best attainable score on any validation query:
\begin{equation}
\max_{\kappa'\in\mathcal{K}_{\mathrm{act}}\setminus\{\kappa\}}
r(q,\kappa')
=
\max_{\kappa'\in\mathcal{K}_{\mathrm{act}}}
r(q,\kappa')
\quad
\text{for all } q\in\mathcal{Q}_{\mathrm{val}}.
\label{eq:no-unique-score-contribution}
\end{equation}
Therefore,
\begin{equation}
g_{\mathcal{K}_{\mathrm{act}}\setminus\{\kappa\}}(q)
=
g_{\mathcal{K}_{\mathrm{act}}}(q)
\quad
\text{for all } q\in\mathcal{Q}_{\mathrm{val}}.
\label{eq:single-prune-preserve-score-profile}
\end{equation}
Applying the same argument to all pruning steps in $\Phi$ gives
\begin{equation}
g_{\Phi(\mathcal{K};\mathcal{Q}_{\mathrm{val}})}(q)
=
g_{\mathcal{K}}(q)
\quad
\text{for all } q\in\mathcal{Q}_{\mathrm{val}}.
\label{eq:phi-preserve-score-profile}
\end{equation}
Averaging over $\mathcal{Q}_{\mathrm{val}}$ yields
\begin{equation}
\mathrm{OCov}(\Phi(\mathcal{K};\mathcal{Q}_{\mathrm{val}});\mathcal{Q}_{\mathrm{val}})
=
\mathrm{OCov}(\mathcal{K};\mathcal{Q}_{\mathrm{val}}).
\label{eq:phi-preserve-ocov-score}
\end{equation}

For the capability frontier, let $\mathcal{K}^{\mathrm{cap}}_{t+1}=\mathcal{C}_t\cup\mathcal{V}_t$ and $\mathcal{C}_{t+1}=\Phi(\mathcal{K}^{\mathrm{cap}}_{t+1};\mathcal{Q}_{\mathrm{val}})$.
By Equation~\ref{eq:phi-preserve-ocov-score},
\begin{equation}
\mathrm{OCov}(\mathcal{C}_{t+1};\mathcal{Q}_{\mathrm{val}})
=
\mathrm{OCov}(\mathcal{C}_t\cup\mathcal{V}_t;\mathcal{Q}_{\mathrm{val}})
\geq
\mathrm{OCov}(\mathcal{C}_t;\mathcal{Q}_{\mathrm{val}}).
\label{eq:capability-ocov-monotone-score}
\end{equation}

For the deploy frontier, let $\mathcal{H}_t=\mathcal{U}_t\cap\mathcal{C}_{t+1}$ be the candidate set used to update the deploy frontier, and let $\mathcal{B}'_t=\Phi(\mathcal{B}_t\cup\mathcal{H}_t;\mathcal{Q}_{\mathrm{val}})$ be the candidate deploy frontier.
If the deploy update is accepted, then $\mathcal{B}_{t+1}=\mathcal{B}'_t$.
By Equation~\ref{eq:phi-preserve-score-profile},
\begin{equation}
g_{\mathcal{B}_{t+1}}(q)
=
g_{\mathcal{B}'_t}(q)
=
g_{\mathcal{B}_t\cup\mathcal{H}_t}(q)
\geq
g_{\mathcal{B}_t}(q)
\quad
\text{for all } q\in\mathcal{Q}_{\mathrm{val}}.
\label{eq:deploy-accepted-preserve-score-profile}
\end{equation}
Averaging over $\mathcal{Q}_{\mathrm{val}}$ gives
\begin{equation}
\mathrm{OCov}(\mathcal{B}_{t+1};\mathcal{Q}_{\mathrm{val}})
\geq
\mathrm{OCov}(\mathcal{B}_t;\mathcal{Q}_{\mathrm{val}}).
\label{eq:deploy-accepted-ocov-monotone-score}
\end{equation}

If the deploy update is rejected, ERSkill keeps the previous deploy frontier, i.e., $\mathcal{B}_{t+1}=\mathcal{B}_t$.
Thus,
\begin{equation}
\mathrm{OCov}(\mathcal{B}_{t+1};\mathcal{Q}_{\mathrm{val}})
=
\mathrm{OCov}(\mathcal{B}_t;\mathcal{Q}_{\mathrm{val}}).
\label{eq:deploy-rejected-ocov-unchanged}
\end{equation}
Combining Equations~\ref{eq:deploy-accepted-ocov-monotone-score} and~\ref{eq:deploy-rejected-ocov-unchanged}, we obtain
\begin{equation}
\mathrm{OCov}(\mathcal{B}_{t+1};\mathcal{Q}_{\mathrm{val}})
\geq
\mathrm{OCov}(\mathcal{B}_t;\mathcal{Q}_{\mathrm{val}}).
\label{eq:deploy-ocov-monotone-score}
\end{equation}

Combining Equations~\ref{eq:capability-ocov-monotone-score} and~\ref{eq:deploy-ocov-monotone-score} completes the proof.
\end{proof}

\section{More Implementation Details}
\label{appen:imple_details}

\subsection{Benchmarks}
\label{appen:benchmarks}

To prevent information leakage across data splits, we construct the train, validation, and test sets at the underlying conversation/sample level rather than at the individual QA-example level. Specifically, all QA examples associated with the same conversation or sample ID are assigned to the same split, ensuring that no underlying interaction history is shared across training, validation, and test sets.

\paragraph{LoCoMo~\citep{maharana2024evaluating}}
LoCoMo contains multi-session conversational histories with four question types: multi-hop, temporal, open-domain, and single-hop questions. Following prior work~\citep{fang2025lightmem, kang2025memory}, we exclude adversarial questions. Our splits contain 233 training examples, 152 for validation, and 314 for testing.

\paragraph{LongMemEval~\citep{wu2024longmemeval}}
LongMemEval is a benchmark for evaluating long-term interactive memory in chat assistants. Each instance consists of timestamped user-assistant history sessions, a user question, the question time, and a reference answer or rubric. We use the LongMemEval-S part for evaluation. Our splits contain 205 training examples, 98 for validation, and 197 for testing.

\paragraph{PerLTQA~\citep{du2024perltqa}}
PerLTQA is a personal long-term memory QA dataset designed to evaluate how models use heterogeneous memory sources in conversation. It covers both semantic memory, including world knowledge, profiles, and social relationships, and episodic memory, including events and dialogues. Our splits contain 439 training examples, 272 for validation, and 483 for testing.

% \subsection{More Details}
% \label{appen:other_details}
% We use \texttt{gpt-4o-mini} as the LLM judge in evaluation.
% All experiments are conducted on a server with four RTX PRO6000 Blackwell GPUs.
% The routed-gain margin $\gamma_{\mathrm{route}}$ is set to 0.00, 0.00, and 0.02 for LoCoMo, LongMemEval, and PerLTQA, respectively.
% The compactness tolerance $\xi_{\mathrm{drop}}$ is set to 0.15, 0.15, and 0.05 for LoCoMo, LongMemEval, and PerLTQA, respectively.
% In Appendix~\ref{appen:rollout_cache}, we also describe how to reduce the training cost by treating rollout results as reusable records.

\subsection{More Details}
\label{appen:other_details}

\paragraph{LLM usage and decoding.}
We use \texttt{gpt-4o-mini} as the LLM-as-a-Judge evaluator
for the reported L-J metric.
During skill evolution, rollout scoring, trace analysis, skill design,
and program generation are instead performed with the corresponding
backbone LLM, i.e., Qwen3-Next-80B-A3B-Instruct or GPT-5.4-nano.
% For the main experiments, we use deterministic decoding with
% temperature $T=0$ and top-$p=1$.
% For the repeated-run stability analysis in
% Appendix~\ref{app:stability}, we use $T=0.7$ to evaluate robustness
% under stochastic generation.
At each evolution step, we generate at most three candidate skill
programs for LoCoMo and LongMemEval, and one for PerLTQA.

\paragraph{Router optimization.}
We train the router with Adam using a learning rate of
$1\times10^{-3}$ and a mini-batch size of 16.
After each evolution step, the router is optimized for at most
100 epochs, with early stopping based on the training loss using a
patience of 10 epochs.
The supervision for each query is constructed as a soft target over
the rollout scores of all available skills, using a temperature of
$\tau=0.2$.
To emphasize recently collected evolution experience while retaining
historical supervision, the router training window uses the most recent
40 queries with a recency decay factor of $0.5$.

\paragraph{Frontier and system settings.}
All experiments are conducted on a server with four RTX PRO6000
Blackwell GPUs.
The routed-gain margin $\gamma_{\mathrm{route}}$ is set to
0.00, 0.00, and 0.02 for LoCoMo, LongMemEval, and PerLTQA,
respectively.
The compactness tolerance $\xi_{\mathrm{drop}}$ is set to
0.15, 0.15, and 0.05, respectively.
In Appendix~\ref{appen:rollout_cache}, we further describe how
training cost is reduced by treating rollout results as reusable
records.

\paragraph{Fair baseline comparison.}
All evolving baselines are provided with the same training QAs, gold references, and correctness feedback.
Adaptive-RAG likewise uses the same information to construct its retrieval strategy.

\subsection{Rollout Reuse and Caching}
\label{appen:rollout_cache}

Skill evolution can be expensive if the same query--skill pair is repeatedly executed during oracle evaluation, router training, frontier recomputation, and deploy validation.
In our implementation, we reduce this cost by treating rollout results as reusable records.
The key principle is that retrieval and answer-generation rollouts are only executed when new information is introduced; subsequent training and selection steps operate on cached query--skill results whenever possible.

\paragraph{Batch-level oracle reuse.}
For each training batch $\mathcal{Q}_t$, ERSkill first evaluates every skill in the current capability frontier $\mathcal{C}_t$ on every query.
This produces a batch-level oracle table, where each entry stores the prediction, execution trace, evidence state, and evaluation scores for a query--skill pair.
The same table is then reused to identify oracle-best skills, construct router supervision, record training traces, compute frontier statistics, and collect cost information.
Thus, once a skill has been evaluated on a query in the batch, later components read its stored result rather than rerunning the skill.

\paragraph{Router replay without rerollout.}
After the router is updated, routed training or validation performance is computed by replaying router decisions over existing oracle tables.
Given a query and a candidate skill set, the router selects a skill, and ERSkill retrieves the selected skill's cached rollout payload from the corresponding oracle table.
This avoids an additional rollout for the router-selected skill.
As a result, routed performance estimation only requires lightweight router inference plus table lookup, rather than another retrieval-and-generation execution.

\paragraph{Incremental candidate evaluation.}
When a new skill candidate is generated, ERSkill does not re-evaluate all existing skills.
Instead, it performs a single-skill rollout for the candidate on the relevant training or validation queries, and then merges the candidate results into the existing oracle table.
Frontier recomputation, candidate filtering, and deploy-frontier construction are then performed as offline score and set operations over the merged table.
This turns many repeated evaluations into data reuse over previously collected query--skill scores.

\paragraph{Validation oracle maintenance.}
ERSkill maintains a validation oracle table throughout evolution.
At initialization, the seed skills are evaluated on the validation set.
When new candidates survive train-side recomputation and require validation, only these new candidates are rolled out on the validation queries.
Their results are merged into the maintained validation oracle table, while previous validation rollouts are reused.
This is especially useful for deploy-frontier updates, which repeatedly compare routed performance under different candidate deploy sets.

\paragraph{Content-hash cache for skill evaluation.}
For evaluation routines that score each skill independently, ERSkill uses a skill-content cache.
Each cached skill payload is associated with a hash of the skill markdown content.
A cached rollout is reused only when the current skill content hash matches the cached hash; otherwise, the skill is rerun.
This avoids redundant rollouts when a skill is unchanged, while preventing incorrect reuse when the skill name remains the same but its content has changed.

\paragraph{Memory and router-side caching.}
ERSkill also caches computations that are shared across rollouts.
The structured memory store is built once and loaded from cache during skill execution, avoiding repeated atomization, entity extraction, embedding, relation extraction, and graph construction.
For the router, skill markdown embeddings are cached and refreshed only when new skills appear or existing skill content changes.
The router training window also stores compact supervision records from historical oracle rollouts, allowing later router updates to reuse previous supervision without re-executing the underlying skills.

Overall, these implementation strategies ensure that expensive rollouts are performed mainly for newly introduced query--skill pairs.
Existing frontier skills, validation results, router supervision, and memory representations are reused through oracle tables, incremental merging, replay-based routed evaluation, and content-aware caches.

\subsection{Prompt Template}
\label{appen:prompt_template}

\subsubsection{LLM-based Relation Extraction}
\label{appen:llm_relation_extraction_prompt}

For LLM-based relation extraction, ERSkill asks the LLM to assign a typed relation from a source atom to each candidate neighbor atom.
For LoCoMo and LongMemEval, we additionally include the temporal hint ``Keep in mind that [Source Atom] happened before every [Neighbor Atom] listed here.''; for PerLTQA, this hint is omitted because the memory sources do not always share a comparable temporal order.

\begin{tcolorbox}[promptbox, title={LLM-based Relation Extraction Prompt}]
\begin{Verbatim}[
  fontsize=\small,
  breaklines=true,
  breakanywhere=true,
  breaksymbolleft={},
  breaksymbolright={}
]
Your task is to find the relation from [Source Atom] to each [Neighbor Atom].
<temporal_hint>

For each neighbor, identify whether one of the following relations clearly holds:
1. Changed: when events in [Source Atom] changed to events in [Neighbor Atom]
2. Cause: when events in [Source Atom] caused events in [Neighbor Atom]
3. Reason: when events in [Source Atom] are due to events in [Neighbor Atom]
4. HinderedBy: when events in [Neighbor Atom] can be hindered by events in [Source Atom], and vice versa
5. React: when, as a result of events in [Source Atom], the subject feels as mentioned in [Neighbor Atom]
6. Want: when, as a result of events in [Source Atom], the subject wants events in [Neighbor Atom] to happen

If a neighbor does not clearly belong to any of these relations, output "none".
Choose a relation only if there is clear evidence matching the definition.
Do not make excessive inferences beyond the given atoms.
Pay attention to who the subject is.
Do not confuse the roles of [Source Atom] and [Neighbor Atom].

Source memory:
<source_atom_text>

Neighbor memories:
<neighbor_atom_json_list>

Return JSON only in this format:
{"relations":[{"atom_id":"<neighbor atom id>","label":"Changed|Cause|Reason|HinderedBy|React|Want|none"}]}
\end{Verbatim}
\end{tcolorbox}

\subsubsection{LLM-as-a-Judge Prompt}
\label{appen:llm_judge_prompt}
We summarize the LLM-as-a-judge prompts used for each benchmark in this section.

\begin{tcolorbox}[promptbox, title={LoCoMo Judge Prompt}]
\begin{Verbatim}[
  fontsize=\small,
  breaklines=true,
  breakanywhere=true,
  breaksymbolleft={},
  breaksymbolright={}
]
Your task is to judge whether a generated answer to a question is correct or wrong. You will be given the following data:
    (1) a question (posed by one user to another user), 
    (2) a ’gold’ (ground truth) answer, 
    (3) a generated answer
which you will score as correct or wrong.

The point of the question is to ask about something one user should know about the other user based on their prior conversations.
The gold answer will usually be a concise and short answer that includes the referenced topic, for example:
Question: Do you remember what I got the last time I went to Hawaii?
Gold answer: A shell necklace
The generated answer might be much longer, but you should be generous with your grading - as long as it touches on the same topic as the gold answer, it should be counted as CORRECT. 

For time related questions, the gold answer will be a specific date, month, year, etc. The generated answer might be much longer or use relative time references (like "last Tuesday" or "next month"), but you should be generous with your grading - as long as it refers to the same date or time period as the gold answer, it should be counted as CORRECT. Even if the format differs (e.g., "May 7th" vs "7 May"), consider it CORRECT if it's the same date.

Now it's time for the real question:
Question: <question>
Gold answer: <gold_answer>
Generated answer: <generated_answer>

First, provide a short (one sentence) explanation of your reasoning, then finish with CORRECT or WRONG. 
Do NOT include both CORRECT and WRONG in your response, or it will break the evaluation script.

Just return the score in json format with the key as "score".
Use {{"score": 1}} for CORRECT and {{"score": 0}} for WRONG.
\end{Verbatim}
\end{tcolorbox}

\begin{tcolorbox}[promptbox, title={LongMemEval Judge Prompt for Single-Session User / Single-Session Assistant / Multi-Session}]
\begin{Verbatim}[
  fontsize=\small,
  breaklines=true,
  breakanywhere=true,
  breaksymbolleft={},
  breaksymbolright={}
]
I will give you a question, a correct answer, and a response from a model. Please judge whether the response contains the correct answer. If the response is equivalent to the correct answer or contains all the intermediate steps to get the correct answer, judge it as correct. If the response only contains a subset of the information required by the answer, judge it as incorrect. Return ONLY valid JSON with key "score": return {"score": 1} if the model response is correct; otherwise return {"score": 0}.

Question: <question>

Correct Answer: <correct_answer>

Model Response: <model_response>
\end{Verbatim}
\end{tcolorbox}

\begin{tcolorbox}[promptbox, title={LongMemEval Judge Prompt for Temporal Reasoning}]
\begin{Verbatim}[
  fontsize=\small,
  breaklines=true,
  breakanywhere=true,
  breaksymbolleft={},
  breaksymbolright={}
]
I will give you a question, a correct answer, and a response from a model. Please judge whether the response contains the correct answer. If the response is equivalent to the correct answer or contains all the intermediate steps to get the correct answer, judge it as correct. If the response only contains a subset of the information required by the answer, judge it as incorrect. Do not penalize off-by-one errors for the number of days, weeks, or months. For example, predicting 19 days when the answer is 18 days is still correct. Return ONLY valid JSON with key "score": return {"score": 1} if the model response is correct; otherwise return {"score": 0}.

Question: <question>

Correct Answer: <correct_answer>

Model Response: <model_response>
\end{Verbatim}
\end{tcolorbox}

\begin{tcolorbox}[promptbox, title={LongMemEval Judge Prompt for Knowledge Update}]
\begin{Verbatim}[
  fontsize=\small,
  breaklines=true,
  breakanywhere=true,
  breaksymbolleft={},
  breaksymbolright={}
]
I will give you a question, a correct answer, and a response from a model. Please judge whether the response contains the correct answer. If the response contains previous information along with an updated answer, the response is correct as long as the updated answer is the required answer. Return ONLY valid JSON with key "score": return {"score": 1} if the model response is correct; otherwise return {"score": 0}.

Question: <question>

Correct Answer: <correct_answer>

Model Response: <model_response>
\end{Verbatim}
\end{tcolorbox}

\begin{tcolorbox}[promptbox, title={LongMemEval Judge Prompt for Single-Session Preference}]
\begin{Verbatim}[
  fontsize=\small,
  breaklines=true,
  breakanywhere=true,
  breaksymbolleft={},
  breaksymbolright={}
]
I will give you a question, a rubric for the desired personalized response, and a response from a model. Please judge whether the response satisfies the desired response. The model does not need to reflect all the points in the rubric. The response is correct as long as it recalls and utilizes the user's personal information correctly. Return ONLY valid JSON with key "score": return {"score": 1} if the model response is correct; otherwise return {"score": 0}.

Question: <question>

Rubric: <rubric>

Model Response: <model_response>
\end{Verbatim}
\end{tcolorbox}

\begin{tcolorbox}[promptbox, title={PerLTQA Judge Prompt}]
\begin{Verbatim}[
  fontsize=\small,
  breaklines=true,
  breakanywhere=true,
  breaksymbolleft={},
  breaksymbolright={}
]
I will give you a question about a character's personal long-term memory, a correct answer, and a response from a model. Judge whether the model response correctly recalls and states the key facts from the character memory. If the response is semantically equivalent to the correct answer, judge it as correct even if the wording differs. If the question asks for multiple key facts and the response misses an important part, judge it as incorrect. If the response adds an unsupported or contradictory fact that changes the meaning of the answer, judge it as incorrect. Focus on factual correctness of the recalled profile, relationship, event, or dialogue memory, not stylistic differences. Return ONLY valid JSON with key "score": return {"score": 1} if the response is correct; otherwise return {"score": 0}.

Question: <question>

Correct Answer: <correct_answer>

Model Response: <model_response>
\end{Verbatim}
\end{tcolorbox}

\subsubsection{LLM-based Skill Router}
\label{appen:llm_router_prompt}

\begin{tcolorbox}[promptbox, title={LLM-based Skill Router Prompt}]
\begin{Verbatim}[
  fontsize=\small,
  breaklines=true,
  breakanywhere=true,
  breaksymbolleft={},
  breaksymbolright={}
]
Choose the single best retrieval skill for the question.
Choose only from the provided skills.
Available retrieval skills:
<skill_candidates>

Question: <question>

Return only one skill name from this set:
<allowed_skill_names>
\end{Verbatim}
\end{tcolorbox}

\subsection{Calculation of Relative Improvement}
\label{app:relative_improvement}

For each evaluation metric
$m \in \{\mathrm{F1}, \mathrm{B1}, \mathrm{L\text{-}J}\}$,
we first identify the strongest baseline according to its average score across the evaluated benchmarks:
$b_m^{*}=\max_{b \in \mathcal{B}} s_m^b$,
where $s_m^b$ denotes the average score of baseline $b$ on metric $m$.

The relative improvement of ERSkill on metric $m$ is then computed as
\begin{equation}
\Delta_m
=
\frac{
s_m^{\mathrm{ERSkill}} - b_m^{*}
}{
b_m^{*}
}
\times 100\% .
\label{eq:metric_relative_improvement}
\end{equation}

Finally, the overall relative improvement is obtained by averaging the metric-wise relative improvements:
$\Delta_{\mathrm{overall}}=\frac{1}{3}\sum_{m \in \{\mathrm{F1}, \mathrm{B1}, \mathrm{L\text{-}J}\}}\Delta_m$.

\section{More Experiment Results}

\subsection{Controlled Comparison under a Shared Memory Substrate}
\label{app:same-memory}

\input{table/comparison_shared_memory_full}

We further disentangle the effects of the memory substrate and the
retrieval policy through a controlled comparison.
Specifically, we construct variants of Dynamic Cheatsheet,
ReasoningBank, and GEPA by replacing their original RAG memory with
the same multi-primitive memory substrate used by ERSkill, while
retaining their original evolution and answer-generation procedures.
Thus, each pair differs primarily in the memory interface available
for retrieval, while the method-specific adaptation mechanism remains
unchanged.

Table~\ref{tab:same-memory-full} reports the complete results across
all three benchmarks and both backbone LLMs.
The effect of the shared substrate varies across methods and settings.
For example, with Qwen3-Next-80B-A3B-Instruct, the average L-J score
decreases from 51.09 to 48.91 for Dynamic Cheatsheet and from 52.72
to 48.08 for ReasoningBank, while GEPA improves from 52.23 to 53.20.
A similar mixed pattern appears with GPT-5.4-nano: Dynamic Cheatsheet
and GEPA improve to 53.32 and 55.83 average L-J, respectively, whereas
ReasoningBank decreases to 49.00.

Across these controlled variants, ERSkill remains substantially
stronger, reaching average L-J scores of 61.30 and 61.11 with the
two backbones. These results show that exposing a richer set of
retrieval primitives does not by itself guarantee better memory use.
The effectiveness of the memory substrate depends on how its available
retrieval operations are composed and selected, which motivates
ERSkill's skill-based retrieval and routing formulation.

% \subsection{Deeper analysis of router and skill behavior}
% \label{app:analysis_of_route_and_skill}

% Table~\ref{tab:router-skill-behavior} shows that query-adaptive routing
% outperforms every fixed single-skill policy, while selections remain
% distributed across multiple skills rather than collapsing to a single
% strategy.
% This suggests that no individual skill is uniformly optimal and that
% the router benefits from matching different retrieval behaviors to
% different query demands.
% Primitive usage is also selective: routed executions use only a subset
% of the available retrieval and processing primitives.
% This is consistent with the role of the primitive library as a search
% space for evolution, while frontier updates retain only skills that
% provide non-redundant value for deployment.
% Together, these results indicate that ERSkill learns both
% query-dependent skill selection and a compact set of complementary
% retrieval behaviors.

\subsection{Additional Details on Evaluator Robustness}
\label{app:evaluator-robustness}

\begin{table*}[t]
\centering
\caption{
Examples illustrating semantic agreement between human and
LLM-based correctness judgments despite differences in wording
or temporal specificity.
}
\small
\setlength{\tabcolsep}{3pt}
\begin{tabular}{p{0.28\textwidth} p{0.28\textwidth} p{0.15\textwidth} cc}
\toprule
\textbf{Query} & \textbf{Prediction} & \textbf{Gold Answer}
& \textbf{Human} & \textbf{GPT-4o-mini} \\
\midrule
Which major holiday season coincides with Evan's wedding?
& Christmas season (around late December 2023)
& Christmas
& 1 & 1 \\
\midrule
When did Evan have his sudden heart-palpitation incident that really shocked him?
& Last week before 6 June 2023, likely late May to early June 2023
& First week of June 2023
& 1 & 1 \\
\bottomrule
\end{tabular}
\label{tab:judge-examples}
\end{table*}

For human evaluation, we randomly sample 50 outputs from the same
LoCoMo test run used in Table~\ref{tab:evaluator-robustness} and
annotate each response with a binary correctness label with respect
to the question and reference answer.

Table~\ref{tab:judge-examples} provides representative examples where
the generated answer is semantically consistent with the reference
despite differing in wording or temporal specificity.

\subsection{Hyperparameter Study}
\label{app:hyperparameter}

In this section, we study the train batch size, which controls the granularity of skill evolution.
ERSkill runs one pass over the training set.
A larger train batch provides more rollout evidence for each evolution step, making frontier recomputation and router updates more stable, but it also reduces the number of evolution steps.
Conversely, a smaller batch allows more frequent skill updates, but each update is based on less reliable rollout statistics.
As shown in Figure~\ref{fig:hyperparameter_batch_size}, a batch size of 20 achieves a good balance between stable per-step evolution and sufficient update frequency.

\subsection{Stability Across Repeated Runs}
\label{app:stability}

\input{table/stability}

We further evaluate the run-to-run stability of ERSkill through five
independent runs across both LLM backbones and all three benchmarks.
Table~\ref{tab:stability-full} reports the mean, standard deviation,
and coefficient of variation (CV) for F1, B1, and L-J.
The results show consistently low variation across settings: all CVs
remain below 6.4\%, and most are below 5\%.
The variation is slightly larger on LongMemEval, while LoCoMo and
PerLTQA are generally more stable.
Overall, the results indicate that ERSkill maintains consistent
performance across repeated evolution runs, backbones, and benchmarks.
The results reported in Section~\ref{sec:comparison} are taken from one of the five runs included in this stability analysis.
As performing multiple repeated runs for all methods would incur substantial computational costs, we primarily conduct the repeated-run analysis on ERSkill.

% The main experiments in Section~\ref{sec:comparison} use deterministic decoding
% ($T=0$, top-$p=1$), while the five-run stability analysis uses
% $T=0.7$ to evaluate robustness under stochastic generation.

\section{Examples}

\subsection{Example of Experience Trie}
\label{appen:experience_trie_example}

Figure~\ref{fig:trie_sample} shows an example of experience trie.

\begin{figure*}[t]
    \centering
    \includegraphics[width=\textwidth]{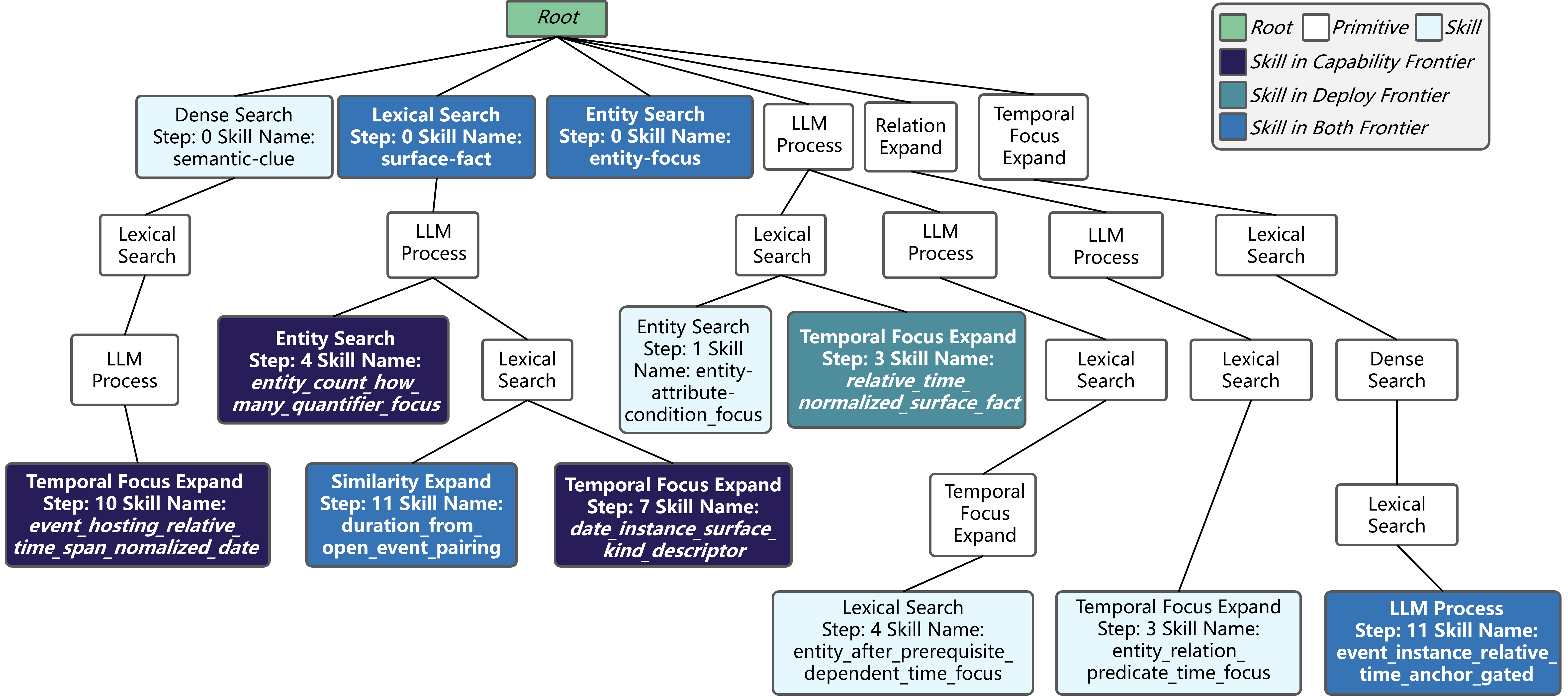}
    \caption{
        The experience trie visualizes explored primitive compositions. Each node denotes a primitive, and colors distinguish retained capability/deploy skills from explored but displaced ones. For proposed skills, ``Step'' indicates the proposal step, and ``Skill Name'' gives the assigned name.
    }
    \label{fig:trie_sample}
\end{figure*}

\section{Limitations and Future Work}
\label{appen:limitations_and_future_work}

While ERSkill provides a retrieval-centric approach to adaptive agent
memory, several directions remain for further development.
First, the current evolution procedure relies on rollout-based
evaluation and LLM-as-a-Judge supervision during training, introducing
additional training-time cost.
More efficient evaluators and selective rollout strategies could further
improve the efficiency of retrieval-skill evolution.
Second, ERSkill adopts a fixed primitive library to provide a stable and
reusable search space for skill evolution.
While this design facilitates experience accumulation and controlled
skill composition, extending the framework with verified primitive
discovery could enable a broader range of retrieval behaviors.
Finally, our experiments focus on long-term memory question answering.
Extending adaptive retrieval skills to broader agent settings, such as
planning, tool use, personalization, and interactive decision-making,
is a promising direction for future work.

%% file: table/comparison_shared_memory_full.tex
\begin{table*}[t]
\centering
\caption{
Full controlled comparison under a shared memory substrate.
``+ Shared substrate'' replaces the method's original RAG memory
with the same multi-primitive memory substrate used by ERSkill,
while retaining its original evolution and answer-generation
procedures. We report F1, BLEU-1 (B1), and LLM-as-a-Judge (L-J)
scores as percentages. Best and second-best results within each
backbone are shown in \textbf{bold} and \underline{underline},
respectively. $^*$ denotes the transfer setting used for ERSkill,
where the LoCoMo-trained system is directly evaluated on LongMemEval
without additional training.
}
\setlength{\tabcolsep}{3.0pt}
\resizebox{\textwidth}{!}{%
\fontsize{8pt}{10pt}\selectfont
\begin{tabular}{lcccccccccccc}
\toprule
\multirow{2}{*}{\textbf{Methods}}
& \multicolumn{3}{c}{\textbf{LoCoMo}}
& \multicolumn{3}{c}{\textbf{LongMemEval$^*$}}
& \multicolumn{3}{c}{\textbf{PerLTQA}}
& \multicolumn{3}{c}{\textbf{Average}} \\
\cmidrule(lr){2-4}
\cmidrule(lr){5-7}
\cmidrule(lr){8-10}
\cmidrule(lr){11-13}
& \textbf{F1} & \textbf{B1} & \textbf{L-J}
& \textbf{F1} & \textbf{B1} & \textbf{L-J}
& \textbf{F1} & \textbf{B1} & \textbf{L-J}
& \textbf{F1} & \textbf{B1} & \textbf{L-J} \\
\midrule

\rowcolor{qwenbar}
\multicolumn{13}{c}{\textbf{Qwen3-Next-80B-A3B-Instruct}} \\
\midrule

Dynamic Cheatsheet
& \underline{34.96} & \underline{33.22} & 56.05
& \underline{33.59} & \underline{34.13} & 51.26
& 35.50 & 40.88 & 45.96
& \underline{34.68} & \underline{36.08} & 51.09 \\

\quad + Shared substrate
& 33.60 & 31.40 & 56.36
& 23.70 & 26.98 & 41.11
& 45.69 & 37.85 & \underline{49.27}
& 34.33 & 32.08 & 48.91 \\[1pt]

ReasoningBank
& 15.77 & 9.12 & 55.41
& 17.02 & 12.85 & 54.31
& 43.05 & 37.55 & 48.44
& 25.28 & 19.84 & 52.72 \\

\quad + Shared substrate
& 16.39 & 9.81 & \underline{57.32}
& 13.02 & 8.92 & 38.07
& 44.25 & 35.85 & 48.86
& 24.55 & 18.19 & 48.08 \\[1pt]

GEPA
& 10.35 & 4.58 & 57.00
& 11.55 & 7.15 & 52.28
& 42.64 & 37.36 & 47.41
& 21.51 & 16.36 & 52.23 \\

\quad + Shared substrate
& 17.61 & 11.16 & 56.36
& 21.53 & 17.27 & \underline{55.83}
& \underline{47.76} & \underline{41.41} & 47.41
& 28.97 & 23.28 & \underline{53.20} \\

\midrule
\rowcolor{my_approach}
\textbf{ERSkill}
& \textbf{49.96} & \textbf{50.18} & \textbf{70.06}
& \textbf{46.79} & \textbf{53.31} & \textbf{59.39}
& \textbf{51.89} & \textbf{44.07} & \textbf{54.45}
& \textbf{49.55} & \textbf{49.19} & \textbf{61.30} \\

\midrule
\rowcolor{gptbar}
\multicolumn{13}{c}{\textbf{GPT-5.4-nano}} \\
\midrule

Dynamic Cheatsheet
& 18.33 & 12.57 & 59.55
& 21.19 & 18.31 & 43.14
& \underline{43.44} & \underline{36.46} & 50.72
& \underline{27.65} & \underline{22.45} & 51.14 \\

\quad + Shared substrate
& 18.20 & \underline{13.37} & \underline{64.96}
& 20.87 & \underline{18.72} & 44.16
& 41.22 & 32.92 & 50.83
& 26.76 & 21.67 & 53.32 \\[1pt]

ReasoningBank
& 17.65 & 11.88 & 57.96
& 18.88 & 13.34 & 49.74
& 43.14 & 36.16 & 48.44
& 26.56 & 20.46 & 52.05 \\

\quad + Shared substrate
& 15.86 & 9.69 & 64.01
& 16.39 & 13.03 & 32.48
& 41.24 & 34.60 & 50.51
& 24.50 & 19.11 & 49.00 \\[1pt]

GEPA
& \underline{18.80} & 12.31 & 59.55
& \underline{21.74} & 18.58 & 50.25
& 40.41 & 36.15 & \underline{51.34}
& 26.98 & 22.35 & 53.71 \\

\quad + Shared substrate
& 17.99 & 11.51 & 64.01
& 5.85 & 3.20 & \underline{52.79}
& 42.05 & 34.05 & 50.69
& 21.96 & 16.25 & \underline{55.83} \\

\midrule
\rowcolor{my_approach}
\textbf{ERSkill}
& \textbf{38.68} & \textbf{33.46} & \textbf{72.93}
& \textbf{38.89} & \textbf{41.01} & \textbf{56.35}
& \textbf{45.35} & \textbf{38.40} & \textbf{54.04}
& \textbf{40.97} & \textbf{37.62} & \textbf{61.11} \\

\bottomrule
\end{tabular}%
}
\label{tab:same-memory-full}
\vspace{-0.5em}
\end{table*}

%% file: table/stability.tex
\begin{table*}[t]
\centering
\caption{
Stability of ERSkill over five independent runs across two LLM
backbones and three benchmarks.
Mean and standard deviation (Std.) are reported on the percentage
scale, with Std. measured in percentage points; CV denotes the
coefficient of variation.
}
\setlength{\tabcolsep}{8pt}
\resizebox{\textwidth}{!}{
\begin{tabular}{llccc|ccc|ccc}
\toprule
\multirow{2}{*}{\textbf{Backbone}}
& \multirow{2}{*}{\textbf{Statistic}}
& \multicolumn{3}{c|}{\textbf{LoCoMo}}
& \multicolumn{3}{c|}{\textbf{LongMemEval}}
& \multicolumn{3}{c}{\textbf{PerLTQA}} \\
\cmidrule(lr){3-5}
\cmidrule(lr){6-8}
\cmidrule(lr){9-11}
&
& \textbf{F1} & \textbf{B1} & \textbf{L-J}
& \textbf{F1} & \textbf{B1} & \textbf{L-J}
& \textbf{F1} & \textbf{B1} & \textbf{L-J} \\
\midrule

\multirow{3}{*}{Qwen}
& Mean
& 49.17 & 49.62 & 68.60
& 46.43 & 55.13 & 56.44
& 50.63 & 42.61 & 53.49 \\

& Std.
& 1.12 & 1.23 & 1.32
& 2.53 & 2.70 & 3.24
& 1.21 & 1.07 & 1.34 \\

& CV (\%)
& 2.28 & 2.48 & 1.93
& 5.46 & 4.90 & 5.74
& 2.38 & 2.50 & 2.51 \\

\midrule

\multirow{3}{*}{GPT}
& Mean
& 38.20 & 33.69 & 70.38
& 38.36 & 40.32 & 54.72
& 45.80 & 38.84 & 53.25 \\

& Std.
& 1.36 & 1.82 & 1.98
& 2.45 & 1.91 & 2.62
& 0.43 & 0.42 & 1.20 \\

& CV (\%)
& 3.56 & 5.39 & 2.81
& 6.38 & 4.74 & 4.80
& 0.94 & 1.07 & 2.25 \\

\bottomrule
\end{tabular}
}
\label{tab:stability-full}
\vspace{-0.5em}
\end{table*}